\documentclass{article}

\usepackage{microtype}
\usepackage{graphicx}
\usepackage{caption}
\usepackage{subcaption}
\usepackage{booktabs} 
\usepackage{bigstrut,multirow,rotating}
\usepackage{ragged2e} 
\usepackage{makecell,tabularx}

\usepackage{hyperref}

\usepackage[accepted]{icml2024}

\usepackage{amsmath}
\usepackage{amssymb}
\usepackage{mathtools}
\usepackage{amsthm}
\usepackage{authblk}

\usepackage[capitalize,noabbrev]{cleveref}

\theoremstyle{plain}

\theoremstyle{definition}

\theoremstyle{remark}

\usepackage[textsize=tiny]{todonotes}
\usepackage{thmtools, thm-restate}

\icmltitlerunning{Debiased Offline Representation Learning for Fast Online Adaptation in Non-stationary Dynamics}

\begin{document}

\twocolumn[
\icmltitle{Debiased Offline Representation Learning for Fast Online Adaptation in Non-stationary Dynamics}




\icmlsetsymbol{equal}{*}
\icmlsetsymbol{correspond}{$\dagger$}


\begin{icmlauthorlist}
\icmlauthor{Xinyu Zhang}{equal,lab,nju}
\icmlauthor{Wenjie Qiu}{equal,lab,nju}
\icmlauthor{Yi-Chen Li}{equal,lab,nju,company}
\icmlauthor{Lei Yuan}{lab,nju,company}
\icmlauthor{Chengxing Jia}{lab,nju,company} \\
\icmlauthor{Zongzhang Zhang}{lab,nju}
\icmlauthor{Yang Yu}{lab,nju,company}
\end{icmlauthorlist}

\icmlaffiliation{lab}{National Key Laboratory for Novel Software Technology, Nanjing University, China}
\icmlaffiliation{nju}{School of Artificial Intelligence, Nanjing University, China}
\icmlaffiliation{company}{Polixir Technologies}


\icmlcorrespondingauthor{Zongzhang Zhang}{ zzzhang@nju.edu.cn}

\icmlkeywords{Offline Meta Reinforcement Learning, Representation Learning, Non-stationary Environments, Information Bottleneck }

\vskip 0.3in
]



\printAffiliationsAndNotice{\icmlEqualContribution} 

\begin{abstract}
Developing policies that can adapt to non-stationary environments is essential for real-world reinforcement learning applications. Nevertheless, learning such adaptable policies in offline settings, with only a limited set of pre-collected trajectories, presents significant challenges. A key difficulty arises because the limited offline data makes it hard for the context encoder to differentiate between changes in the environment dynamics and shifts in the behavior policy, often leading to context misassociations. To address this issue, we introduce a novel approach called \textbf{D}ebiased \textbf{O}ffline \textbf{R}epresentation learning for fast online \textbf{A}daptation (\textbf{DORA}). DORA incorporates an information bottleneck principle that maximizes mutual information between the dynamics encoding and the environmental data, while minimizing mutual information between the dynamics encoding and the actions of the behavior policy. We present a practical implementation of DORA, leveraging tractable bounds of the information bottleneck principle. Our experimental evaluation across six benchmark MuJoCo tasks with variable parameters demonstrates that DORA not only achieves a more precise dynamics encoding but also significantly outperforms existing baselines in terms of performance.
\end{abstract}

\section{Introduction}

In Reinforcement Learning (RL), the agent attempts to find an optimal policy that maximizes the cumulative reward obtained from the environment ~\citep{rlbook}. However, online RL training typically requires millions of interaction steps \cite{mbAtari}, posing a risk to safety and cost constraints in real-world scenarios \cite{SafeRL_survey}. Different from online RL, offline RL \cite{BCQ, BRAC, CQL, prdc} aims to learn an optimal policy, exclusively from datasets pre-collected by certain behavior policies, without further interactions with the environment. Showing great promise in turning datasets into powerful decision-making machines, offline RL has attracted wide attention recently~\citep{rl_tutorial}.

Previous offline RL methods commonly assume that the learned policy will be deployed to environments with stationary dynamics~\citep{CQL}, while unavoidable perturbations in real world will lead to \emph{non-stationary} dynamics~\cite{RL_in_nonstationary_environments}. For example, the friction coefficient of the ground changes frequently when a robotic vacuum walks on the floor covered by different surfaces. Policies trained for stationary dynamics will be unable to deal with non-stationary ones since the optimal behaviors depend on the certain dynamic. Nevertheless, training an adaptable policy from offline datasets for online adaptation in non-stationary dynamics is overlooked in current RL studies. In this paper, we focus on the setting where the environment dynamics, such as gravity or damping of the controlled robot, change unpredictably within an episode.

Offline Meta Reinforcement Learning (OMRL), which trains a generalizable meta policy from a multi-task dataset generated in different dynamics, offers a potential solution to handle non-stationary dynamics. Among existing OMRL methods, the gradient-based approaches~\cite{MerPO} may not be ideal solutions, as extra gradient updates are likely to cause drastic performance degradation due to the instability of policy gradient methods ~\citep{SAC}. In comparison, context-based OMRL methods \cite{CORRO, FOCAL} extract task-discriminative information from trajectories by learning a context encoder. However, these methods still face severe challenges when encountering non-stationary dynamics. Firstly, the representations from the learned encoder may exhibit a biased correlation with the data-collecting behavior policy~\citep{CORRO}. As a consequence, it will fail to correctly identify the environment dynamics when the learned policy, instead of the behavior policy, is used to collect context during online adaptation. Secondly, existing OMRL methods generally require collecting short trajectories before evaluation during meta-testing phrase~\cite{FOCAL}, which is not allowed in non-stationary dynamics because the changes of dynamics are unknown to the agent.

To tackle the aforementioned issues, we propose \textbf{DORA} (abbreviation of \textbf{D}ebiased \textbf{O}ffline \textbf{R}epresentation learning for fast online \textbf{A}daptation). DORA employs a context encoder that uses the most recent state-action pairs to infer the current dynamics. In order to learn a debiased task representation, we respectively derive a lower bound to maximize the mutual information between the dynamics encoding and environmental data, and an upper bound to minimize mutual information between the dynamics encoding and the actions of the behavior policy, following the Information Bottleneck (IB) principle. Specifically, the lower bound is derived to urge the encoder to capture the task-relevant information using InfoNCE~\citep{InfoNCE}. And the upper bound is for debiasing representations from behavior policy, formalized as the Kullback-Leibler (KL) divergence between the representations with and without behavior policy information of each timestep. The contextual policy is then trained by an offline RL algorithm, such as CQL~\citep{CQL}. Experiment results on six MuJoCo tasks with three different changing parameters demonstrate that DORA remarkably outperforms existing OMRL baselines. Additionally, we illustrate that the learned encoder can swiftly identify and adapt to frequent changes in environment dynamics. We summarize our main contributions as below:
\begin{itemize}
\item We propose an offline representation learning method for non-stationary dynamics, which enables fast adaptation to tasks with frequent dynamics changes. 
\item We derive a novel objective for offline meta learning, which trains the encoder to reduce the interference of the behavior policy to correctly identify dynamics.
\item Experimentally, our method achieves better performance in unseen dynamics on-the-fly without pre-collecting trajectories, compared with baselines.
\end{itemize}
\section{Preliminaries}
\paragraph{Reinforcement Learning}

We consider the infinite-horizon Markov Decision Process (MDP) ~\citep{rlbook}, which can be formulated as a tuple $M = \langle\mathcal{S}, \mathcal{A}, P, r, d_0, \gamma\rangle$. Here, $\mathcal{S}$ and $\mathcal{A}$ represent the state space and action space, respectively. Let $\Delta_X$ be the set of probability measures over any space $X$. We use $P: \mathcal{S}\times\mathcal{A}\to \Delta_\mathcal{S}$ to denote the environment transition function, or equivalently, dynamics\footnote{In this paper, we use ``transition function'' and ``dynamics'' interchangeably.}. $r: \mathcal{S}\times \mathcal{A}\to \mathbb{R}$ is the reward function. $d_0$ denotes the distribution of initial states. $\gamma \in \left[0,1\right)$ is the discount factor. The discounted cumulative reward starting from time step $t$ is defined as $G_t=\sum_{k=0}^\infty \gamma^k r\left(s_{t+k}, a_{t+k}\right)$, where $G_0$ is known as return. The policy $\pi: \mathcal{S} \to \Delta_\mathcal{A}$ specifies a distribution over the action space $\mathcal{A}$, given any state $s \in \mathcal{S}$. RL aims to find an optimal policy $\pi^*$ which maximizes the expected return $J_M(\pi)$, i.e., $\pi^* \in \arg\max_{\pi} J_M(\pi) := \mathbb{E}[G_0]$. When executing the policy $\pi(a|s)$ starting from $s_t$, the value function can be formulated as $V^\pi(s)=\mathbb{E}_\pi\left[G_0 | s_0=s\right]$. Similarly, The action value function is $Q^\pi(s, a)=$ $\mathbb{E}_\pi\left[G_0 \mid s_0=s, a_0=a\right]$.

\paragraph{Context-based OMRL}
OMRL combines offline RL~\citep{CQL} and meta RL~\citep{PEARL}, aiming to learn a meta policy from offline datasets collected from multiple tasks. The goal of OMRL is to learn a meta policy $\pi_{\theta}$ with $\theta$ denoting the learnable parameters, which can generalize to unseen tasks not appearing in the training dataset. The $N$ tasks for meta training $\mathcal{M} = \{M_i\}_{i=1}^N$ are sampled from the same task distribution $P_{\rm train}$. We assume that all tasks share the same state space and action space, with the only difference lying in dynamics. For each task $M_i$, we have an offline dataset $D_i$ of trajectories collected by an unknown behavior policy $\pi^b_i$. For simplicity, we may omit the subscript $i$. In this paper, we use $\tau_{t-H:t}= \{a_{t-H}, s_{t-H+1} \ldots, a_{t-1}, s_{t}\}$ to denote a short offline trajectory of length $H$ for dynamics identification. During OMRL training, the meta policy $\pi_{\theta}$ is trained on all the offline datasets. In the testing phase, the agent encounters unseen tasks over $P_{\rm test}$, and the meta policy needs to generalize to these tasks with limited interactions. As mentioned before, the objective of OMRL is to maximize the expected return in any task sampled from $P_{\rm test}$, i.e., 
\begin{equation}
\underset{\theta}{\max } \  J(\pi_{\theta}):= \mathbb{E}_{M \sim P_{\rm test}}\left[J_M\left(\pi_{\theta}\right)\right].
\end{equation}
Context-based OMRL utilizes an encoder $p_\phi$, where $\phi$ is the learnable parameters, to output a $m$-dimensional task representation $z \in \mathbb{R}^m$. In experiments shown in \cref{sec:exp}, we set $m=2$. The task representation could be regarded as auxiliary information to help the offline policy $\pi_\theta$ to identify the current testing task. For clarity, we use $\pi_\theta(a|s, z)$, $V^{\pi}(s, z)$, $Q^{\pi}(s, a, z)$ thereafter to respectively denote the meta policy, meta value function, and meta action value function which takes the task representation $z$ as an additional input. Compared with the online context-based encoder, the offline-trained encoder is more susceptible to behavior policies due to the limited state and action space coverage of the datasets, which is detrimental in context-based OMRL.

\begin{figure*}[!t]
  \centering
    \includegraphics[width=0.965\linewidth, trim = 6.5cm 5.6cm 5cm 5.2cm, clip]{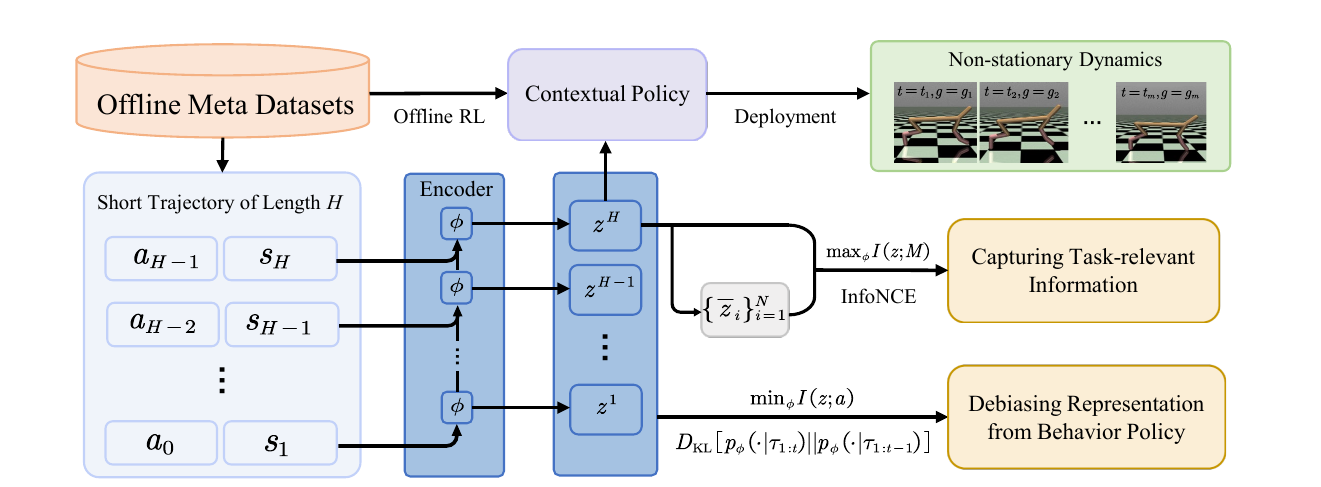}
  \caption{The DORA framework. The encoder utilizes recent state-action pairs to maintain a set of representations $ \{z^{1}, \cdots, z^H \} $ and $z^H$ updates the moving average task encodings $\{\bar{z}_i\}_{i=1}^N$. All these representations are then used to optimize the encoder. The contextual policy is trained through offline RL on the datasets, where each transition is labeled with its representation by the learned encoder. }
  \label{Framework}
\end{figure*}

\paragraph{Information Bottleneck}
In information theory, the IB concept involves extracting essential information by constraining the information flow, thereby finding a delicate equilibrium between preserving target-related information and achieving efficient compression~\cite{IB}. This problem can be framed as the minimization of mutual information $I$ between the source random variable $X$ and its representation $\hat{X}$ for efficient compression, alongside the maximization of mutual information between $\hat{X}$ and the target $Y$ to retain target-related information \cite{DLandIB}. The pursuit of an optimal $\hat{X}$ is cast as the following maximization problem:
\begin{equation} \label{IB_loss}
\text{max}_{\hat{X}} I(\hat{X};Y)  - \beta  I(X;\hat{X}).
\end{equation}
Here, $I(\hat{X};Y)=\int\int p(\hat{x}, y) \log\left(\frac{p(\hat{x}, y)}{p(\hat{x})p(y)}\right) \,d\hat{x} \,dy$, where $p(\hat{x},y)$ is the joint probability density function, and $p(\hat{x})$ and $p(y)$ are marginal probability density functions. The hyper-parameter $\beta$ represents the tradeoff between compression and fitting. IB enables the context encoder to learn abstract and high-level representations, which helps to improve the generalization ability of the meta policy on unseen tasks \cite{CVAE, betaTCVAE}.

\section{Method}
\label{Method}

In this section, we present the DORA framework, which learns a dynamics-sensitive trajectory encoder that mitigates biases from the behavior policy and effectively adapts to non-stationary dynamics. In \cref{The Information Bottleneck Objective}, we adhere to the information bottleneck principle to formulate the objective for offline meta-encoder learning. Subsequently, in \cref{maxMI} and \cref{minMI}, we convert IB into two tractable losses. \cref{algorithm} elaborates on the proposed DORA framework (see \cref{Framework}) including its implementation details.

\subsection{From IB to Debiased Representations} \label{The Information Bottleneck Objective}

Due to the fact that the offline dataset, generated by fixed behavior policies, exhibits limited coverage of the state and action space, the encoder tends to inaccurately treat the behavior policy as a prominent feature for task identification. Such a predicament can result in the decline of the performance of the contextual policy, given its dependency on the encoder's responsiveness to shifts in environmental dynamics. Thus it is crucial for the encoder to discern tasks based on the dynamics rather than the behavior policy.

We follow the IB principle to tackle this issue. Specifically, our approach involves maximizing the mutual information between representations and tasks to encapsulate dynamics-relevant information, while simultaneously minimizing the mutual information between representations and actions of behavior policies to alleviate biases stemming from the behavior policy. Given a training task $M \sim P_{\rm train}$ and the offline trajectory $\tau$ collected on $M$ with trajectory length $H$, this idea can be formulated as:
\begin{equation} \label{objective}
\text{max}_{\phi} I(z; M)-\beta I(z; a),
\end{equation}
where $z \sim p_\phi(\cdot | \tau)$, $a \sim \pi^b(\cdot|s)$ and $s\in \mathcal{S}$.

\subsection{Contrastive Dynamics Representation Learning} \label{maxMI}

The first term in \cref{objective}, $I(z; M)$, is used for the learned representation to effectively encapsulate dynamics-relevant information. Since directly optimizing $I(z; M)$ is intractable, we derive a lower bound of it as shown in \cref{proposition1}, which formulates the problem of maximizing $I(z; M)$ as a contrastive learning objective. By training the encoder to distinguish between positive and negative samples, the encoder acquires the capability to generate representations that correctly embed the dynamics information.

\begin{restatable}[]{theorem}{lowerb}\label{proposition1}
Denote a set of $N$ tasks as $\mathcal{M}$, in which each task $M_i$ is sampled from the same training task distribution $P_{\rm train}$. Let random variables $M \in \mathcal{M}$, $\tau$ be a trajectory collected in $M$, $z \sim p_\phi(\cdot|\tau)$, $p(z)$ is the prior distribution of $z$, then we have
\begin{align*} \label{theorem_loss1}
I(z ; M) \geq \mathbb{E}_{M, \tau, z} \left[\log \frac{\frac{p_\phi\left(z \mid \tau \right)}{p\left(z\right)}}{\sum_{M_i\in \mathcal{M}} \frac{p_\phi\left(z \mid \tau^i\right)}{p\left(z\right)}}\right] + \log N,
\end{align*}
where $\tau^i$ is a trajectory collected in task $M_i$, $z_i \sim p_\phi(\cdot|\tau^i)$, and $i \in \{1,2,\cdots, N\}$.
\end{restatable}

Due to space limitations, we defer the proof of \cref{proposition1} to \cref{Proof of Theorems}. In practice, to approximate $p_{\phi}(z|\tau)/p(z)$, we choose the radius basis function $S(z_{i}, z_{j}) = \mathrm{exp}(-||z_{i} - z_{j}||^2/{\alpha})$, which measures the similarity between the two representations $z_i$ and $z_j$~\citep{mnih2012fast}. More details about this approximation can be found in \cref{Explanations of the Approximation in Distortion Loss}. Here, $\alpha$ is a hyper-parameter. We then obtain the distortion loss $\mathcal{L}_{\text{Dist}}$ as follows:
\begin{equation} \label{pratical_L1}
\mathcal{L}_{\text{Dist}}(\phi) = -\sum_{\substack{M_i \in \mathcal{M}, \\ \tau^i \in D_i }}\left[\mathrm{log} \left(\frac{S\left(z_{i}, \bar{z}_{i}\right)}{\sum_{j=1}^N S\left(z_i,\bar{z}_j\right)}\right)\right],
\end{equation}
where $z_{i}$ is encoded by the encoder $p_{\phi}$ from trajectory $\tau^i$. $\bar{z}_i$ denotes the average task representation for task $i$, updated by 
$
\bar{z}_i \leftarrow \lambda z_{i} + (1 - \lambda) \bar{z}_{i},
$
where $\lambda\in (0,1)$ is a hyper-parameter. Intuitively, by grouping trajectories from the same task while distinguishing those from different tasks, the distortion loss will help the encoder to extract dynamics-relevant information.

\subsection{ Debiasing Representation from Behavior Policy} \label{minMI}
Ideally, the learned representation should be minimally influenced by the behavior policy. Hence, as shown in the second term in \cref{objective}, we minimize the mutual
information $I(z;a)$ between representations and actions of behavior policy. However, it is intractable to precisely calculate $I(z;a)$. We thus consider deriving an upper bound of it. Firstly, we have the following theorem.

\begin{restatable}[]{theorem}{upperb}\label{proposition2}
Given a training task $M \in \mathcal{M}$, a trajectory $\tau$ collected in $M$ and $z\sim p_\phi(\cdot|\tau)$, we have
\begin{equation}\label{equ:ub} 
I(z ; a) \leqslant \mathbb{E}_a \left[ \  D_{\mathrm{KL}}\left[p\left(\cdot \mid a \right) \| t(\cdot)\right] \right],
\end{equation}
where $t(z)$ is an arbitrary distribution over $\mathbb{R}^m$ and $D_{\rm KL}$ is the Kullback-Leibler divergence~\citep{kld}.
\end{restatable}

The proof of \cref{proposition2} is deferred to \cref{Proof of Theorems}. To convert \cref{equ:ub} into a tractable optimization objective, we choose $p_\phi(\cdot|\tau_{1:t-1})$ to be the prior distribution $t(z)$ at $t-1$ in practice. Since $\tau_{1:t}$ has two additional terms ($a_{t-1}$ and $s_t$) compared with $\tau_{1:t-1}$, we surprisingly find that $p(\cdot|a)$ in \cref{equ:ub} can be cast to $p_\phi(\cdot|\tau_{1:t})$. This is because $a_{t-1}$ generated by $\pi^b(\cdot|s_{t-1})$ carries the information of the behavior policy. Besides, $s_t$, which is sampled from $P(\cdot|s_{t-1}, a_{t-1})$, is independent of the behavior policy. We thus get the following debias loss $\mathcal{L}_{\rm Debias}$:
\begin{equation} \label{pratical_L2}
\mathcal{L}_{\rm Debias}(\phi)  = \sum_{\substack{M_i \in \mathcal{M}, \\ \tau^i_{1:H} \in D_i, \\ t \in \{2, \cdots, H\}} }D_{\mathrm{KL}}\left[p_\phi(\cdot|\tau^i_{1:t}) \| p_\phi(\cdot|\tau^i_{1:t-1})\right],
\end{equation}

For task $M_i$, we can get a trajectory segment $\tau_{1:H}$ with length $H$ given a full trajectory randomly sampled from $D_i$.

\begin{algorithm}[!t]
    \caption{DORA Training}
    \label{Context-encoder Training}
    \begin{algorithmic}
        \STATE {\bfseries Input:} context encoder $p_{\phi}$, offline datasets $\{D_{i}\}_{i = 1}^{N}$, short trajectory length $H$, batch size $U$, contextual policy $\pi_\theta$, contextual action function $Q_\psi$
       \STATE {\color{gray} //~Context Encoder Training}
       \FOR{step = 1, 2, $\cdots$}
           \FOR{$u = 1$ to $U$}
               \STATE Randomly sample a task $i \in \{1, 2, \cdots, N\}$
               \STATE Sample a trajectory $\tau^{u}$ with length $H$ from $D_i$ randomly
           \ENDFOR
           \STATE Let $\mathcal{B}_1 = \{\tau^{u}\}_{u = 1}^{U}$
           \STATE Infer representation $z_u$ from $\tau^u$ with $p_\phi$, $\forall \tau^u \in \mathcal{B}_1$
           \STATE Update the moving average task representation $\{\bar{z}_i\}_{i = 1}^{N}$
           \STATE Get encoder loss $\mathcal{L}_{\mathrm{DORA}}(\phi)$ from \cref{total_loss}
           \STATE Update encoder $p_{\phi}$ with $\phi \leftarrow \phi - \alpha \nabla_\phi\mathcal{L}_{\mathrm{DORA}}(\phi)$
       \ENDFOR
       \STATE {\color{gray}//~Policy Training}
       \FOR{step = 1, 2, $\cdots$}
           \STATE Randomly get a minibatch $\mathcal{B}_2 = \{\tau^{u}\}_{u = 1}^{U}$
           \STATE Infer representation $z_u$ from $\tau^u$ with $p_\phi$, $\forall \tau^u \in \mathcal{B}_2$
           \STATE Augment the states of minibatch $\mathcal{B}$ with $\{z_u\}_{u = 1}^{U}$
           \STATE Update $\pi_\theta $ and $Q_\psi$ with CQL~\citep{CQL} using $\{z_u\}_{u = 1}^{U}$ on $\mathcal{B}_2$
       \ENDFOR
    \end{algorithmic}
\end{algorithm}

\subsection{ A Practical Implementation } \label{algorithm}

We take the Recurrent Neural Network (RNN) as the backbone of the encoder \cite{GRU}. The encoder receives a historical sequence of state-action pairs to generate the representation. By restricting the length of the RNN, we can adjust the amount of historical trajectory information retained in the encoder \cite{ESCP}, which is helpful when the dynamic changes. Specifically, let $H$ be the history length, then the representation at step $t$ is denoted as $z_i^t \sim p_{\phi}(\cdot | \tau_{t-H:t})$, where $\tau_{t-H:t}= \{a_{t-H}, s_{t-H+1}, \ldots, a_{t-1}, s_t\}$. If the history length is less than $H$, we will use the zero padding trick~\citep{GRU}. The overall architecture is shown in \cref{Framework}.

During the meta-training process, the encoder $p_{\phi}$ is updated using the the following loss:
\begin{equation} \label{total_loss}
\mathcal{L}_{\text{DORA}}(\phi) = \mathcal{L}_{\text{Dist}}(\phi)+ \beta \mathcal{L}_{\text{Debias}}(\phi).
\end{equation}
After the encoder is trained to convergence, we train the meta policy via CQL~\cite{CQL}. During meta-testing, the encoder infers the task representation $z$ based on the trajectories collected up to date. At each step, the current representation $z$ and state $s$ are together input into the meta policy, yielding a policy adapted to the current dynamics. Different from existing approaches~\citep{CORRO, FOCAL}, our approach does not require to pre-collect trajectories before policy evaluation, which achieves fast online adaptation on-the-fly. To sum up, the pseudocodes of training and testing are illustrated in \cref{Context-encoder Training} and \cref{DORA testing}, respectively. We release the code at Github\footnote{\url{https://github.com/LAMDA-RL/DORA}.}.

\section{Related Work}
In this section, we introduce works related to our framework. \cref{related work: Offline Reinforcement Learning} provides an overview of related works in offline RL. \cref{related work: RL in Non-stationary Dynamics} shows how online RL methods handle non-stationary dynamics. Finally, \cref{related work: Offline Meta Reinforcement Learning} delves into the realm of OMRL, exploring relevant research and focusing on the challenges and strategies in this domain.

\subsection{Offline RL} \label{related work: Offline Reinforcement Learning}
Offline RL learns policies from the dataset generated by certain behavior policies without online interaction with the environment, which helps avoid the safety concerns of online RL. The key challenge of offline RL lies in the distribution shift problem \cite{levine2020offline}, thus algorithms in this domain typically impose constraints on the training policy to keep it close to the behavior policy. Many model-free approaches \cite{BCQ, BRAC, CQL,ghosh2022offline, prdc} introduce regularization terms by considering the divergence between the learned policy and the behavior policy to constrain policy deviation. CQL \cite{CQL} directly learns a conservative state-action value function to alleviate the overestimation problem. Model-based methods \cite{MOPO, Morel, COMBO, swazinna2022user} pre-train an environment model from the dataset, then use the model to generate out-of-dataset predictions for state-action transitions. 

\subsection{RL in Non-stationary Dynamics} \label{related work: RL in Non-stationary Dynamics}
 
To overcome the unavoidable perturbations in real world, recent years have seen an increasing number of research into RL in non-stationary dynamics, including continual RL and meta RL methods. Continual learning approaches aim to enable agents to continuously learn when faced with unseen tasks and avoid forgetting old tasks \cite{ContinualRL_survey}. Policy consolidation \cite{kaplanis2019policy} integrates the policy network with a cascade of hidden networks and uses history to regularize the current policy to handle changing dynamics. 
The meta learning methods \cite{MetaRL_survey}, including the gradient-based methods and context-based methods, adapt to new tasks quickly by leveraging experience from training tasks. Specifically, the gradient-based meta RL methods \cite{MAML} update the policy with gradients in the testing tasks, which are not suitable for non-stationary dynamics since real-world tasks may not allow for extra updates of the gradient. Context-based meta RL infers task-relevant information by learning a context encoder. In these methods, PEARL \cite{PEARL} employs variational inference to learn a context encoder. ESCP \cite{ESCP} utilizes an RNN-based context encoder to rapidly perceive changes in dynamics, enabling fast adaptation to new environments. 

However, both continual RL and online context-based RL methods require online interactions, which are infeasible in the offline setting. As far as we know, there is currently no method capable of learning an effective adaptive policy in non-stationary dynamics in the paradigm of offline RL.

\begin{table*}[ht]
\centering
\caption{Performance on the MuJoCo tasks in \textbf{stationary dynamics}. \textbf{Top:} Average normalized return $\pm$ standard deviation over 5 random seeds in testing tasks with \textbf{IID dynamics}. \textbf{Bottom:} Average normalized return $\pm$ standard deviation over 5 random seeds in testing tasks with \textbf{OOD dynamics}.}

\begin{subtable}[]{\textwidth}
\centering
\small{
\begin{tabular}{c r r r r r}
\toprule[1.0pt]
    \multicolumn{1}{c}{Environment} & \multicolumn{1}{c}{Offline ESCP} & \multicolumn{1}{c}{FOCAL} & \multicolumn{1}{c}{CORRO} & \multicolumn{1}{c}{Prompt-DT} & \multicolumn{1}{c}{\textbf{DORA (Ours)}} \bigstrut\\
    \midrule

    Cheetah-gravity  &  78.22 $\pm$ 21.13  &  51.99 $\pm$ 10.44   & 56.70 $\pm$ 15.48 &  49.59 $\pm$ 18.44& \textbf{86.31} $\pm$ 16.45\\
    Pendulum-gravity &  96.55 $\pm$ 17.24  &  74.95 $\pm$ 12.55  &  76.82 $\pm$ 11.39 & 34.71 $\pm$ 15.95 & \textbf{100.08} $\pm$ \,\,\,0.01  \\
    Walker-gravity   &  52.82 $\pm$ 21.01  &  17.22 $\pm$ 12.05  & 44.49 $\pm$ 24.67  & 26.09 $\pm$ \,\,\,4.38 & \textbf{66.87} $\pm$ 21.64  \\ 
    Hopper-gravity   &   \textbf{78.70} $\pm$ 19.17    &  34.10 $\pm$ 11.12     &    76.73 $\pm$ 14.19   & 42.44 $\pm$ 22.19 & 74.68 $\pm$ 17.34  \\
    Cheetah-dof      &  93.90 $\pm$ 24.49  &  39.36 $\pm$ \,\,\,9.27 & 53.80 $\pm$ 13.90 & 43.99 $\pm$ 16.23 & \textbf{97.85} $\pm$ 12.30\\
    Cheetah-torso    &  56.21 $\pm$ 12.37  &  39.35 $\pm$ \,\,\,6.33   &  51.93 $\pm$ 16.22 & 45.30 $\pm$ \,\,\,5.60 & \textbf{61.60} $\pm$ \,\,\,1.19 \\
    
\bottomrule[1.0pt]
\end{tabular}
}
\end{subtable}
\\[3pt]

\begin{subtable}[]{\textwidth}
\centering
\small{
\begin{tabular}{c r r r r r}
\toprule[1.0pt]
    \multicolumn{1}{c}{Environment} & \multicolumn{1}{c}{Offline ESCP} & \multicolumn{1}{c}{FOCAL} & \multicolumn{1}{c}{CORRO} & \multicolumn{1}{c}{Prompt-DT} & \multicolumn{1}{c}{\textbf{DORA (Ours)}} \bigstrut\\
    
    \midrule
    
    Cheetah-gravity &  64.30 $\pm$ 24.69    &  42.84 $\pm$ 12.14 & 38.72 $\pm$ 14.78 &  30.19 $\pm$ \,\,\,7.74 & \textbf{70.05} $\pm$ 17.02\\
    Pendulum-gravity &    86.02 $\pm$ 12.26  &  29.53 $\pm$ 14.41  &  57.21 $\pm$ 18.36 & 20.89 $\pm$ 10.99 & \textbf{98.09} $\pm$ 13.95  \\
    Walker-gravity &   35.64 $\pm$ 20.12  &  12.69 $\pm$ \,\,\,7.08  & 29.02 $\pm$ 19.60  & 18.91 $\pm$ 11.78 & \textbf{43.33} $\pm$ 14.21  \\ 
    Hopper-gravity &   \textbf{71.96} $\pm$ 24.40    &  12.96 $\pm$ \,\,\,4.63     &    16.70 $\pm$ \,\,\,7.73   & 33.52 $\pm$ 10.82 & 71.52 $\pm$ 21.37  \\
    Cheetah-dof & 73.60 $\pm$ 22.73  &  27.36 $\pm$ 12.27 & 43.08 $\pm$ 16.09 & 30.77 $\pm$ 10.46 & \textbf{77.48} $\pm$ 12.25\\
    Cheetah-torso &   56.50 $\pm$ 11.65    &    38.13 $\pm$ \,\,\,6.50   &  50.89 $\pm$ \,\,\,6.69 &34.04 $\pm$ \,\,\,4.83 & \textbf{61.57} $\pm$ \,\,\,1.34 \\
    
\bottomrule[1.0pt]
\end{tabular}
}
\end{subtable}

\label{table1}
\end{table*}

\begin{table*}[htbp]
  \centering
  \caption{Average normalized return $\pm$ standard deviation on the MuJoCo tasks in \textbf{non-stationary dynamics} over 5 random seeds.}
  \small{
    \begin{tabular}{crrrrr}
    \toprule
    \multicolumn{1}{c}{Environment} & \multicolumn{1}{c}{Offline ESCP} & \multicolumn{1}{c}{FOCAL} & \multicolumn{1}{c}{CORRO} & \multicolumn{1}{c}{Prompt-DT} & \multicolumn{1}{c}{\textbf{DORA (Ours)}} \bigstrut\\
    \midrule
    Cheetah-gravity &  65.42 $\pm$ \,\,\,6.54    &  49.96 $\pm$ \,\,\,6.41 & 54.08 $\pm$ \,\,\,8.54 &  23.52 $\pm$ \,\,\,7.75& \textbf{71.61} $\pm$ \,\,\,7.56\\
    Pendulum-gravity &    85.92 $\pm$ 16.37  &  \,\,\,8.66 $\pm$ \,\,\,5.42  &  69.66 $\pm$ 14.80 & 11.36 $\pm$ \,\,\,5.86 & \textbf{97.11} $\pm$ 16.91  \\
    Walker-gravity &   27.34 $\pm$ 10.75  &  17.91 $\pm$ 10.76  &  33.62 $\pm$ 14.86  & 14.49 $\pm$ \,\,\,2.81 & \textbf{40.46} $\pm$ 19.04  \\ 
    Hopper-gravity &   \textbf{54.29} $\pm$ 25.39    &  21.09 $\pm$ \,\,\,7.99     &    53.26 $\pm$ 20.36   & 15.91 $\pm$ 10.13 & 49.74 $\pm$ 23.74 \\
    Cheetah-dof & 80.72 $\pm$ \,\,\,8.49  &  37.27 $\pm$ 13.76 & 52.41 $\pm$ 13.99 & 20.24 $\pm$ 12.66 & \textbf{91.83} $\pm$ 10.57\\
    Cheetah-torso &   53.99 $\pm$ 16.09    &    42.23 $\pm$ 15.01   &  45.58 $\pm$ 15.98 &  33.79 $\pm$ 10.51 & \textbf{60.13} $\pm$ 13.65 \\
    
    \bottomrule
    \end{tabular}%
  }
  \label{table2}%
\end{table*}%

\subsection{OMRL} \label{related work: Offline Meta Reinforcement Learning}
OMRL extends meta RL to the paradigm of offline setting, aiming to generalize from experience of training tasks and facilitate efficient adaptation to unseen testing tasks. Gradient-based OMRL methods \cite{MerPO} require a few interactions to adapt to unseen tasks, but the costly online gradient updates may not be allowed in the real world. In comparison, the context-based OMRL methods employ a context encoder for task inference, showing potential for faster adaptation. FOCAL \cite{FOCAL} uses distance metric learning loss to distinguish different tasks. Recently, the transformer-based approaches \cite{Prompt-DT, CMT} leverage Transformer networks for autoregressive training on blended offline data from multiple tasks, and construct policy by planning or setting return-to-go. As mentioned in our work, another key challenge for OMRL is that the encoder is required to accurately identify environment dynamics while debiasing representations from behavior policy. To alleviate this problem, CORRO \cite{CORRO} employs contrastive learning to train context encoders and trains dynamic models to generate new samples. Nonetheless, these works still suffer from the entanglement of task dynamics and behavior policy experimentally. Besides, existing works need to pre-collect trajectories before evaluation in meta-testing phase \cite{CSRO, GENTLE}, which are not able to handle changing dynamics since the changes are unknown to the agent. From the information bottleneck perspective, we tackle these problems by developing a novel model-free framework that efficiently debiases the representations from the behavior policy and swiftly adapts to non-stationary dynamics. 


\section{Experiments}\label{sec:exp}
In this section, we conduct the experiments to answer the following questions:
\begin{itemize}
\item How does the learned encoder and contextual policy of DORA perform in unseen tasks of in-distribution (IID) and out-of-distribution (OOD) dynamics?
\item How well does DORA identify and adapt to online non-stationary dynamics?
\item Are the learned representations of DORA debiased from the behavior policy?
\end{itemize}
In \cref{Environments}, we introduce the environments and baselines. In \cref{Performance}, we compare the performance of different algorithms in IID, OOD, and non-stationary dynamics. In \cref{Representation Visualization} and \cref{Representation Tracking}, we visualize the representations in IID dynamics and study the encoder’s sensitivity to non-stationary
dynamics. In \cref{Debiased Representation Studies}, we study whether the learned representations are debiased from behavior policy. In \cref{Ablation}, we ablate multiple design choices in the encoder training process.

\subsection{Environments and Baselines} \label{Environments}

\textbf{Task Description}. We choose MuJoCo tasks for experiments, including \texttt{HalfCheetah-v3}, \texttt{Walker2d-v3}, \texttt{Hopper-v3}, and \texttt{InvertedDoublePendulum-v2}, which are common benchmarks in offline RL \cite{mujoco}.

\begin{figure}[ht] 
  \centering
  \subcaptionbox{DORA}{
    \includegraphics[width=0.468\linewidth, trim = 0cm 0.6cm 1.5cm 1.0cm, clip]{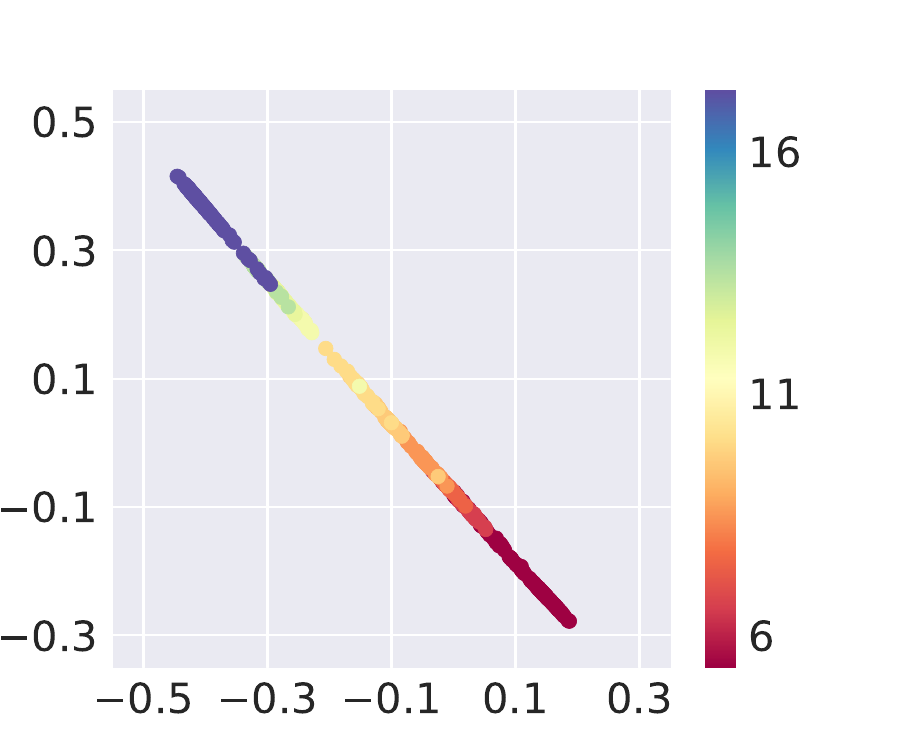}
  }
  \hspace{0.1mm} 
  \subcaptionbox{FOCAL}{
    \includegraphics[width=0.468\linewidth, trim = 0cm 0.6cm 1.5cm 1.0cm, clip]{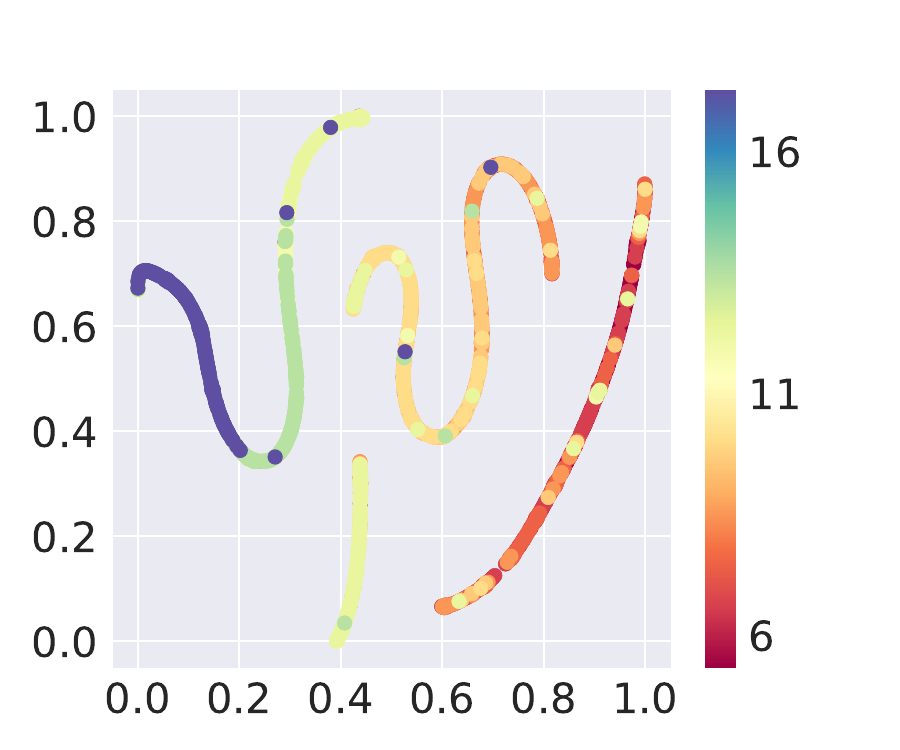}
  }

  \subcaptionbox{CORRO}{%
    \includegraphics[width=0.47\linewidth, trim = 0cm 0.6cm 1.5cm 1.0cm, clip]{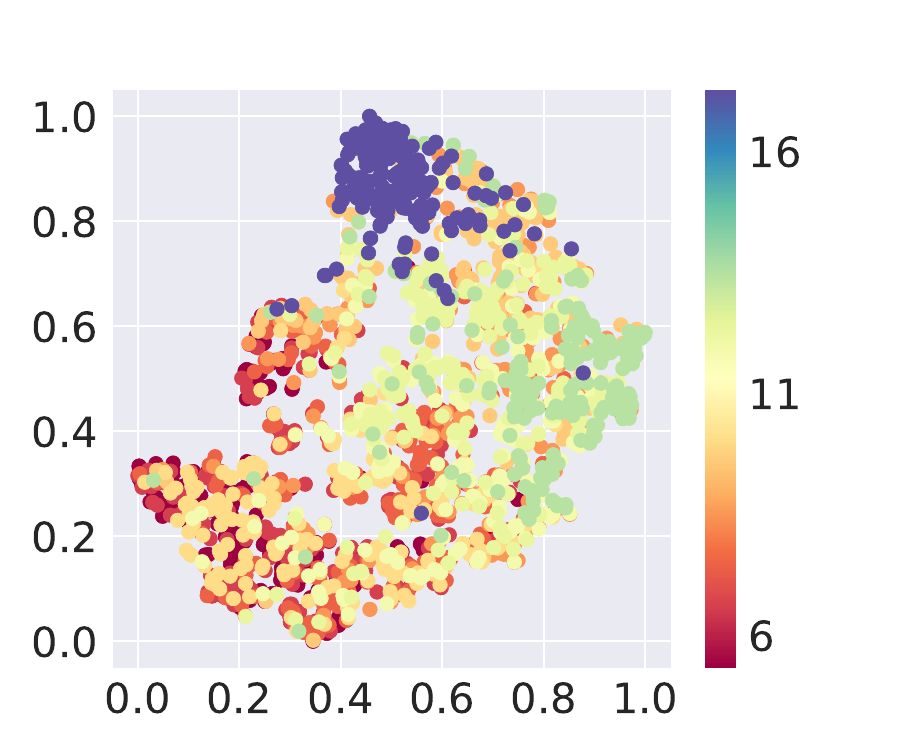}
  } 
  \hspace{0.1mm} 
  \subcaptionbox{Offline ESCP}{
    \includegraphics[width=0.47\linewidth, trim = 0cm 0.6cm 1.5cm 1.0cm, clip]{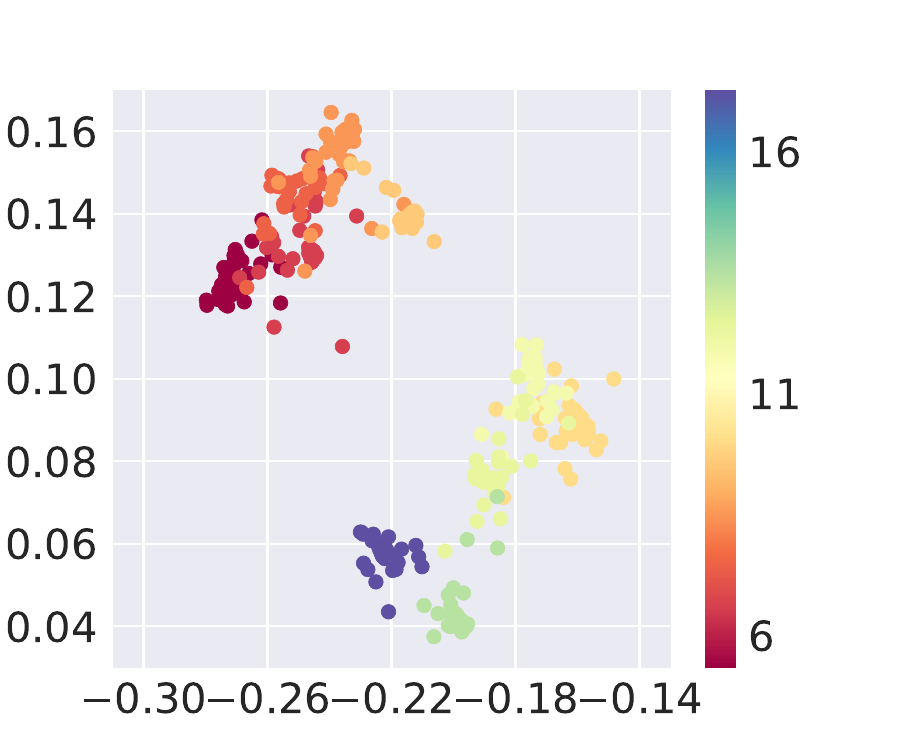}
  } 
\caption{Representation visualization in \texttt{Cheetah-gravity} tasks with IID dynamics. The points are the (projected) representations in a 2D latent space, with the color indicating the real parameters of dynamics.}
\label{Fig: Representation Visualization}
\end{figure}

\textbf{Changing Dynamics}. In order to generate different non-stationary dynamics, we change the physical parameters of the environments, namely \texttt{gravity}, \texttt{dof-damping}, and \texttt{torso-mass}. In \textbf{IID} dynamics, the parameter is uniformly sampled from the same distribution used for sampling training tasks. In \textbf{OOD} dynamics, the parameter is uniformly sampled from a distribution of the 20\% range outside the IID dynamics. The parameter of \textbf{non-stationary dynamics} is sampled from the union of IID and OOD dynamics and changes every 50 timesteps.  For convenience, we will abbreviate the experimental tasks as \texttt{Cheetah-gravity}, \texttt{Pendulum-gravity}, \texttt{Walker-gravity}, \texttt{Hopper-} \texttt{gravity}, \texttt{Cheetah-dof}, and \texttt{Cheetah-torso}.

\textbf{Offline Data Collection}. 
For each training dataset, we use SAC\cite{SAC} to train a policy for each training task independently to make sure the behavior policies are different in every single-task dataset. The offline dataset is then collected from the replay buffer. We gather 200,000 transitions for each single-task dataset, except for the tasks of \texttt{Pendulum-gravity}, which comprises 40,000 transitions.

\textbf{Baselines:} We compare DORA with 4 baselines.
\textbf{CORRO} \cite{CORRO}, \textbf{FOCAL} \cite{FOCAL}, and \textbf{Prompt-DT} \cite{Prompt-DT} are prominent OMRL baselines in recent years. ESCP \cite{ESCP} is an online meta RL approach that effectively adapts to sudden changes. In order to operate this algorithm in the offline setting, we develop the \textbf{offline ESCP} as a baseline.

For all baselines and our method, a history trajectory of length $8$ is used to infer task representations. The debiased loss weight $\beta$ of our method is adjusted for different tasks. More details about the environments, offline datasets, and baselines in our experiments can be found in \cref{Appendix: Experiments Details}.

\begin{figure}[ht]
  \centering
    \includegraphics[width=1.12\linewidth, trim = 4.36cm 0.1cm 0cm 3.0cm, clip]{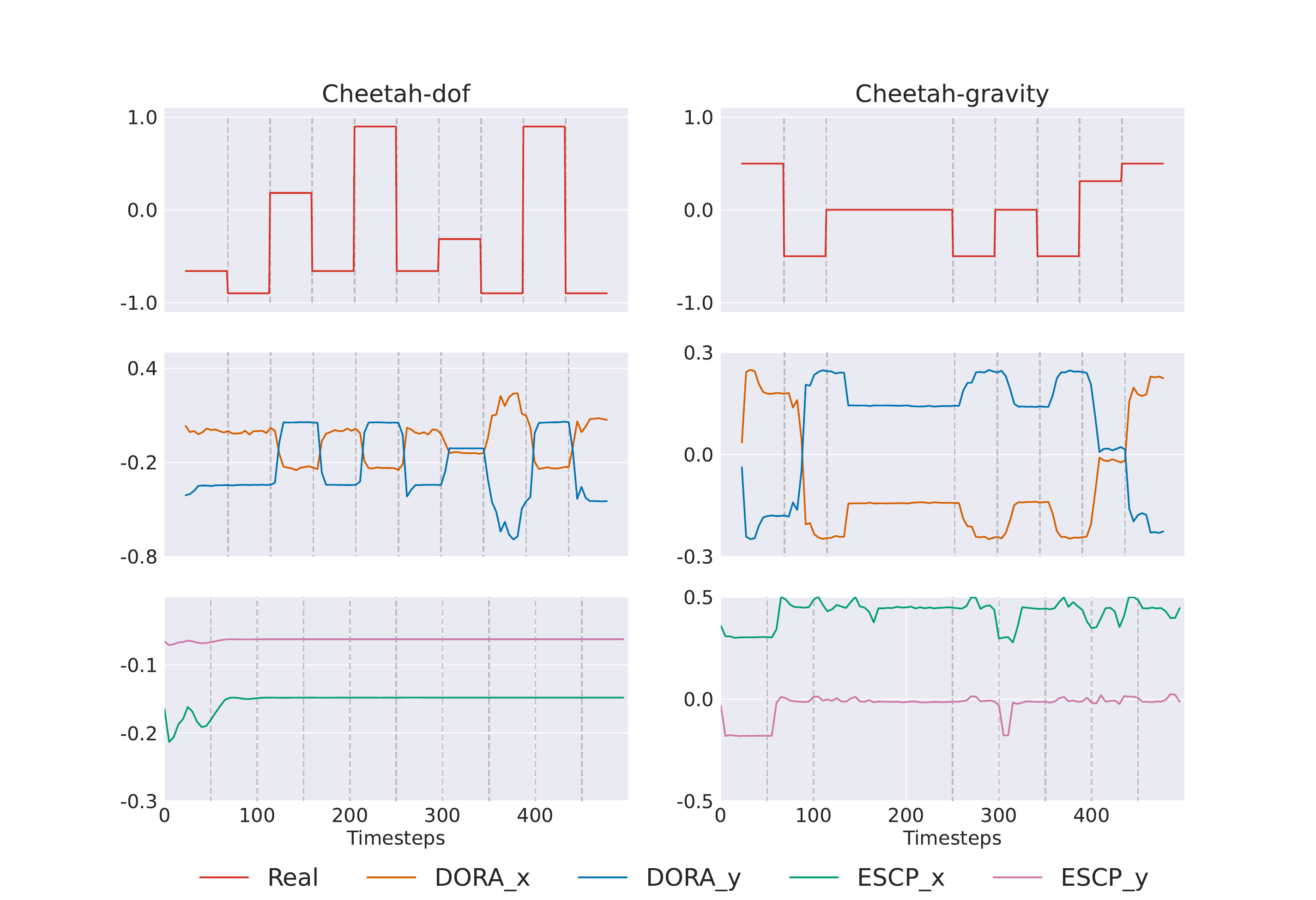}
  \caption{Representation tracking in a single trajectory in non-stationary dynamics. \texttt{Real} represents the normalized real parameters of unseen dynamics, \texttt{DORA\_x} and \texttt{DORA\_y} are the coordinates of the DORA's representations in the 2D latent space, and the same applies to offline ESCP. \textbf{Left:} In \texttt{Cheetah-dof}. \textbf{Right:} In \texttt{Cheetah-gravity}.}
  \label{Fig: Representation Tracking}
\end{figure}

\subsection{Performance} \label{Performance}

We conduct experiments in IID, OOD, and non-stationary dynamics to compare the normalized return of the meta CQL policies learned with encoders from different algorithms. 
During testing, all algorithms are not permitted to pre-collect trajectories on new tasks, which means the encoder has to use short history trajectories generated by contextual policy to infer representations. As shown in \cref{table1},  DORA outperforms the baselines in 5 of the total 6 MuJoCo tasks in both IID and OOD dynamics. Although offline ESCP performs better in \texttt{Hopper-gravity}, its performance suffers from a significant drop of 8.6\% 
from IID to OOD tests, while the drop of DORA is 4.2\%. The poor performance of Prompt-DT reveals that the transformer struggles to unleash its fitting and generalization capabilities in the absence of a well-defined prompt. DORA also excels at handling non-stationary dynamics, and its lead in most of the non-stationary environments is stronger than that in stationary environments, as illustrated in \cref{table2}. These results suggest that DORA is effective in dealing with unseen non-stationary dynamics. As all baselines share the same contextual policies, it can thus be inferred that our encoder is more sensitive to the environment dynamics.


\subsection{Representations in IID Dynamics} \label{Representation Visualization}

In this section, we compare the latent representations of DORA with other context-based OMRL baselines in \texttt{Cheetah-gravity} with IID dynamics, which are visualized in \cref{Fig: Representation Visualization}. The representations of DORA and offline ESCP are encoded into two-dimensional vectors, while the representations of CORRO and FOCAL are encoded into higher dimensions and projected into 2D space via t-SNE \cite{t-SNE}, following the approach used in the original paper.  
\begin{figure}[ht] 
  \centering
  \begin{subfigure}[b]{0.235\textwidth}
    \includegraphics[width=\textwidth, trim = 1cm 0.5cm 2.2cm 1.0cm, clip]{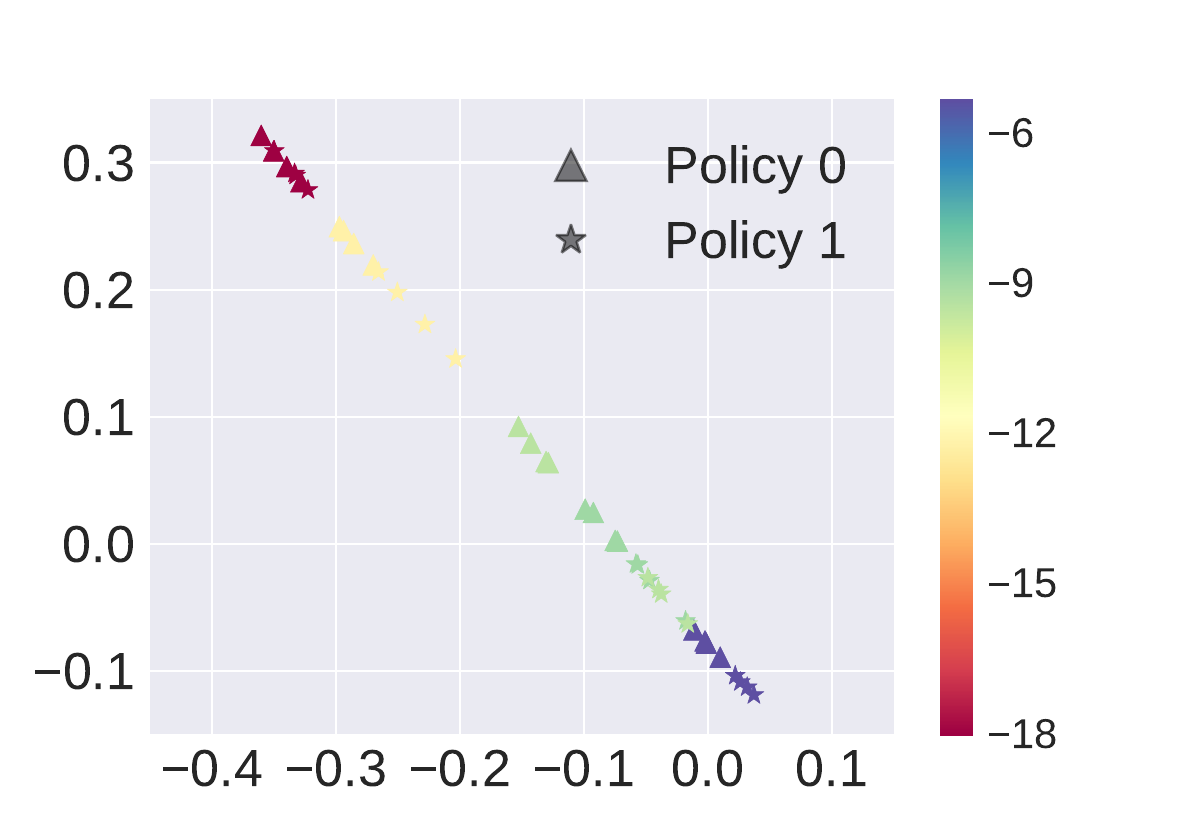}
    \caption{Cheetah-gravity}
    \label{fig: debias1}
  \end{subfigure}
  \hspace{0.05mm}
  \begin{subfigure}[b]{0.235\textwidth}
    \includegraphics[width=\textwidth, trim =  1cm 0.5cm 2.2cm 1.0cm, clip]{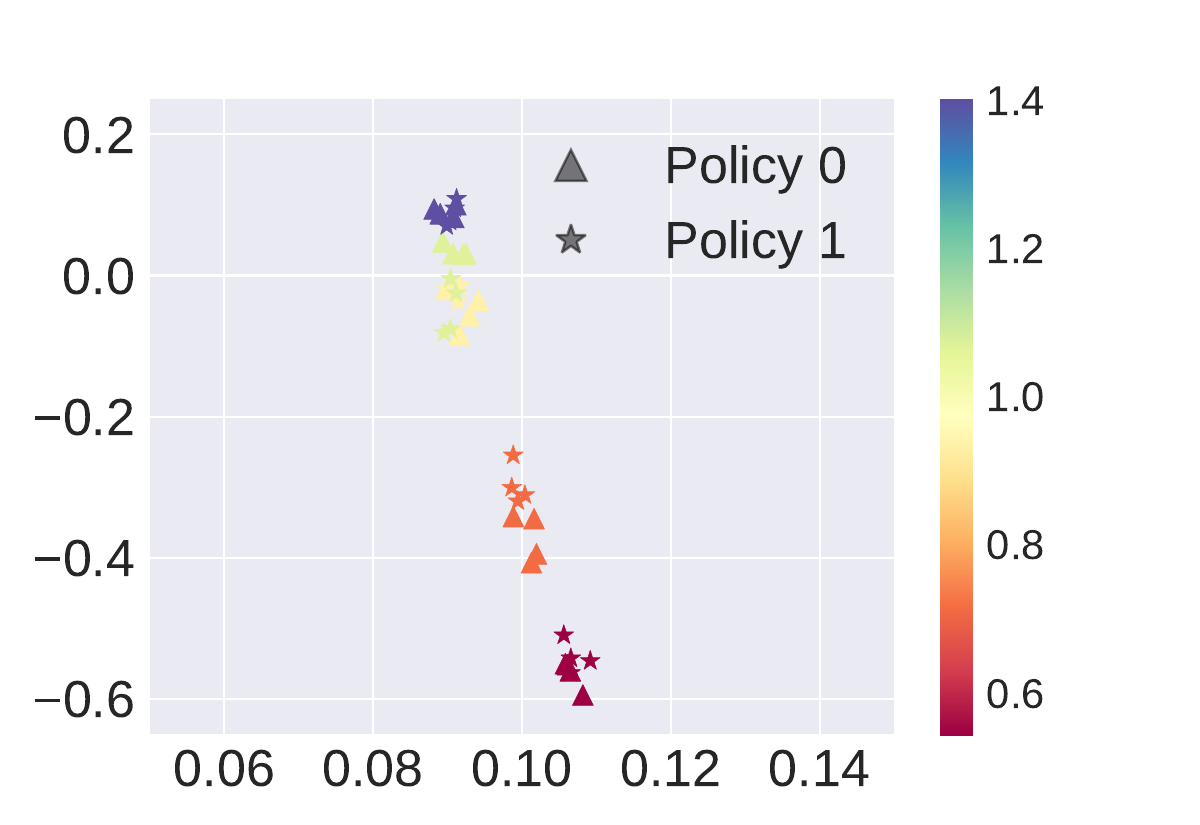}
    \caption{Cheetah-dof}
    \label{fig: debias2}
  \end{subfigure}
\caption{Visualization on task representations generated with 2 different context-collection policies in 5 unseen dynamics. Points of different shapes represent different policies. \textbf{Left:} In \texttt{Cheetah-gravity}. \textbf{Right:} In \texttt{Cheetah-dof}.} 
\label{Fig: debias}
\end{figure}
For each algorithm, we sample 200 short trajectories from the tasks with IID dynamics. The color of the encodings represents the real value of the changing dynamics parameter, and the color bar is on the right of each sub-figure correspondingly. The visualization results suggest that DORA's representations appear as a straight line with parameters of dynamics gradually increasing or decreasing from one end to the other. 
In contrast, representations of CORRO and FOCAL are much more cluttered. Notably, although ESCP extracts linear-shape representations in the online setting, the offline ESCP fails to achieve the same effect. The shape of representations indicates that DORA's encoder extracts more informative encodings and distinguishes the dynamics better than other baselines. We also visualize the encodings in OOD dynamics in \cref{Appendix: Representation Visualization of DORA}, which indicates DORA's representations can be extended to unseen OOD dynamics.

\subsection{Representation Tracking in Changing Dynamics} \label{Representation Tracking}

In order to study the encoder's sensitivity to changing dynamics, we track the changes of representations during the adaptation in part of a single trajectory, which is shown in \cref{Fig: Representation Tracking}. \textit{Real} represents the normalized real parameters of unseen dynamics, \textit{x} and \textit{y} are the coordinates of the representations in the two-dimensional latent space. The environment dynamics change every 50 timesteps. It is evident to find that DORA's encoding promptly follows almost every sudden change of dynamics, while offline ESCP generates few responses. As illustrated by these experimental results, DORA's encoder infers robust representations and responds to non-stationary dynamics swiftly and precisely.

\subsection{Debiased Representation Visualization} \label{Debiased Representation Studies}
In this part, we study whether the learned representations are debiased from behavior policy. We first randomly choose the datasets of training tasks and train a Behavior Clone (BC, \citet{BC}) policy using each chosen dataset. Then these trained BC policies are used to roll out trajectories of $H$ timesteps in different unseen IID and OOD dynamics. Subsequently, we adopt DORA's encoder to generate representations with these short trajectories, which are visualized in \cref{Fig: debias}. 
\begin{figure}[ht] 
  \centering

  \begin{subfigure}[b]{0.23\textwidth}
    \centering
    \includegraphics[width=\textwidth, trim =  1.7cm 0.8cm 2cm 0.5cm, clip]{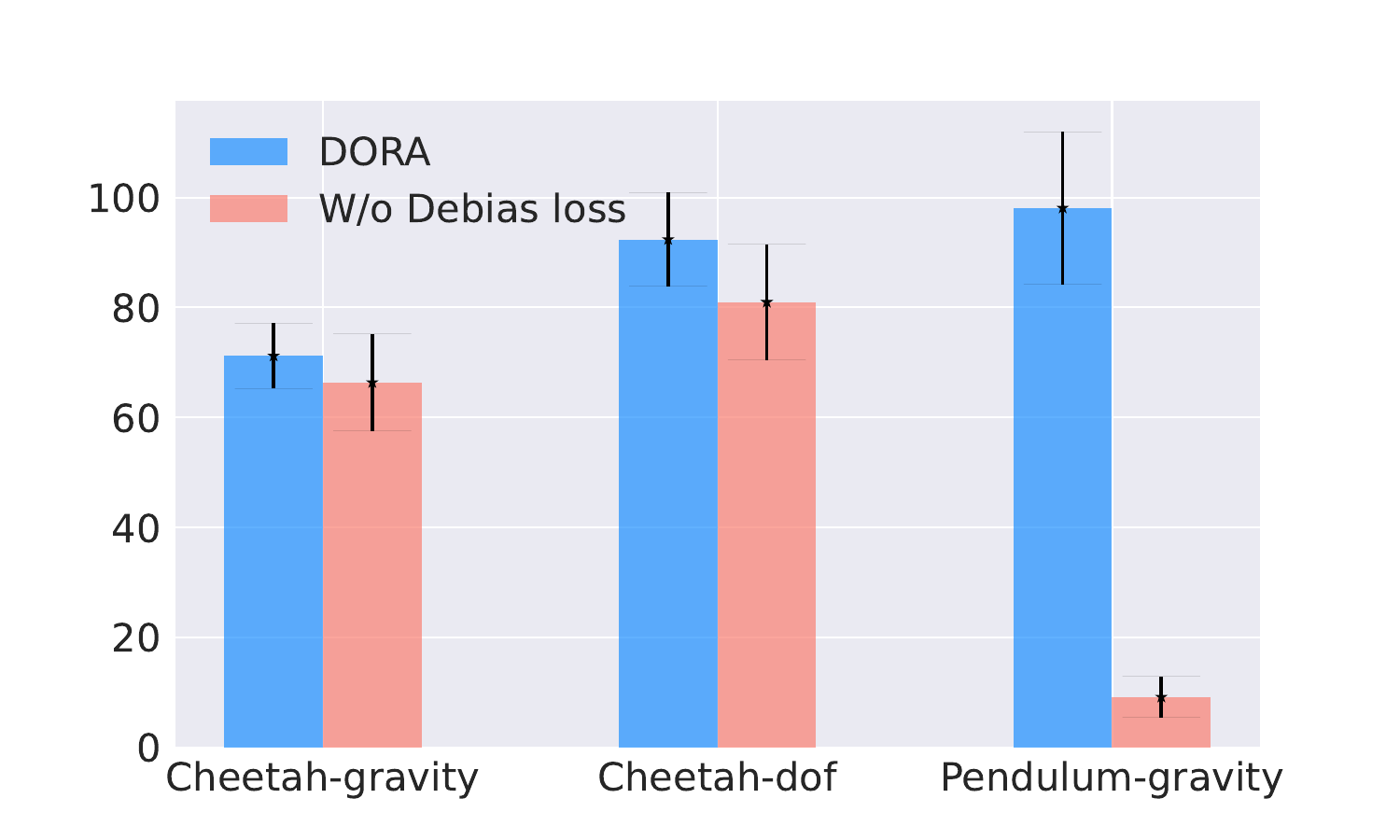}
    \caption{Ablation on debias loss}
    \label{Fig: ablation_loss}
  \end{subfigure}
  \hspace{0.05mm}
  \begin{subfigure}[b]{0.23\textwidth}
    \centering
    \includegraphics[width=\textwidth, trim =  1.7cm 0.8cm 2cm 0.5cm, clip]{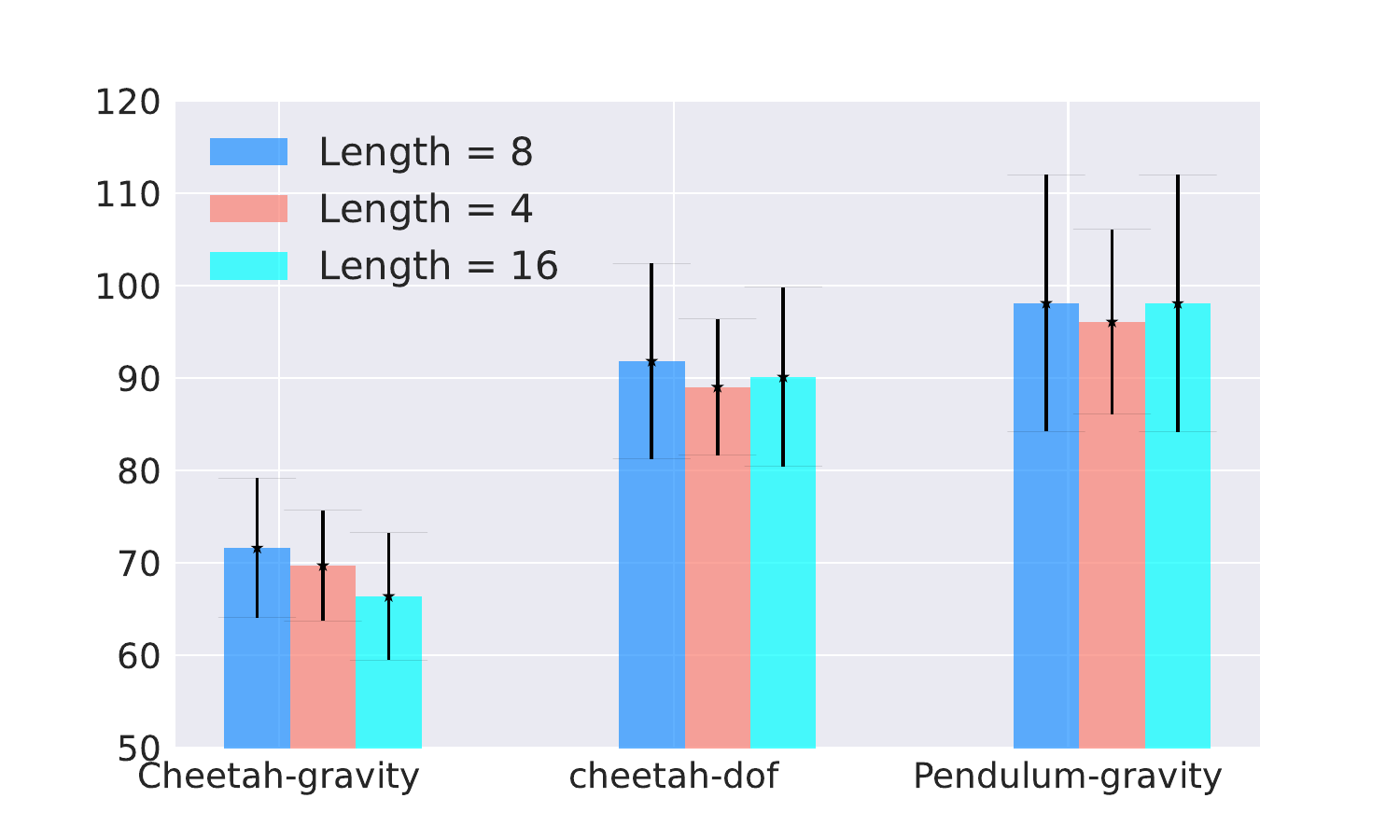}
    \caption{Ablation on history length}
    \label{Fig: ablation_length}
  \end{subfigure}
  
\caption{Ablation studies: Average normalized return of DORA on 3 environments over 5 random seeds. The error bar stands for the standard deviation. \textbf{Left:}  DORA with and without the debias loss.  \textbf{Right:} RNN history lengths of 4, 8, and 16. }
\label{Fig: Ablation}
\end{figure}
The experimental results show that representations inferred from trajectories generated by different policies still cluster together in the same task dynamics. Moreover, in comparison with the results of DORA in \cref{Fig: Representation Visualization}, these representations are encoded to similar coordinates according to the real parameter of dynamics. These results powerfully demonstrate that DORA learns a dynamic-sensitive encoder, which takes the environment dynamics rather than behavior policy as a feature to generate debiased representations.

\subsection{Ablation Studies } \label{Ablation}
We first conduct ablation experiments to study the interdependence of the distortion loss and the debias loss.  In \cref{Fig: ablation_loss}, the ablation experiments in 3 randomly chosen environments show that the debias loss significantly improves DORA. Notably, DORA can not converge without the distortion loss, so this ablated version is not included in the bar chart. Besides, to assess the impact of RNN history length on DORA's performance, we compared instances with RNN lengths of 4, 8, and 16. The experimental results in \cref{Fig: ablation_length} demonstrate that DORA is not sensitive to the RNN history length. Additionally, we conducted sensitivity tests on the hyper-parameter $\beta$ in \cref{Appendix: Ablation Studies on beta}.

\section{Summary}
In this paper, we consider the online adaptation issue in non-stationary dynamics from a fully offline perspective. We propose a novel offline representation learning framework DORA, in which the context encoder can efficiently identify changes in task dynamics and the meta policy is capable of rapidly adapting to online non-stationary dynamics. We follow the information bottleneck principle to formalize the offline representation learning problem. To conquer this problem, we derive a lower bound for the encoder to capture task-relevant information and an upper bound to debias the representation from behavior policy. We compare DORA with other OMRL baselines in environments of IID, OOD, and non-stationary dynamics. Experimental results show that DORA debiases the representation from behavior policy and exhibits better performance compared with baselines.

\section*{Impact Statement}
Our work focuses on offline representation learning, we anticipate that its societal impact lies in advancing the field of Machine Learning. As our work primarily involves theoretical advancements and algorithmic innovations, we do not foresee specific ethical or safety concerns associated with our contributions. The potential societal impact of our work lies in fostering more efficient and adaptable offline RL policies, which can benefit various applications across industries. We believe that the implications of our research align with well-established ethical standards in the field of machine learning.

\section*{Acknowledgements}

The authors would like to thank Fanming Luo, Chen-Xiao Gao and the anonymous reviewers for their support and helpful discussions on improving the paper. 


\bibliography{example_paper}
\bibliographystyle{icml2024}

\newpage
\appendix
\onecolumn
\section{Proofs and Explanations}
\subsection{Proof of Theorems} \label{Proof of Theorems}

\lowerb*

\begin{proof}
We follow the proof of \citet{CORRO}.
\begin{align*}
I(z ; M) &= \mathbb{E}_{\mathcal{M},z}\left[\log \frac{p(M \mid z)}{p(M)}\right] \\ 
& \overset{(a)}= \mathbb{E}_{\mathcal{M},z}\left[\log \mathbb{E}_\tau\left[\frac{p_\phi(z \mid \tau)}{p(z)}\right]\right] \\
&\geq \mathbb{E}_{\mathcal{M},z} \mathbb{E}_\tau\left[\log \frac{p_\phi(z \mid \tau)}{p(z)}\right] \\
&= \mathbb{E}_{\mathcal{M}, \tau, z}\left[\log \frac{p_\phi(z \mid \tau)}{p(z)}\right] \\
&= - \mathbb{E}_{\mathcal{M}, \tau, z} \left[ \log \left( \frac{p\left(z\right)}{p_\phi\left(z \mid \tau\right)} N \right) \right] + \mathrm{log}N \\
& \overset{(b)} \geq - \mathbb{E}_{\mathcal{M}, \tau, z} \left[ \log \left(1+\frac{p\left(z\right)}{p_\phi\left(z \mid \tau\right)}(N-1)\right) \right] + \mathrm{log}N \\
&= - \mathbb{E}_{\mathcal{M}, \tau, z} \left[ \log \left( 1+\frac{p\left(z\right)}{p_\phi\left(z \mid \tau\right)}(N-1) \underset{M_i\in \mathcal{M} \backslash\{M\}}{\mathbb{E}} [\frac{p_\phi\left(z \mid \tau^i\right)}{p\left(z\right)}]\right) \right] + \mathrm{log}N \\
&\approx - \mathbb{E}_{\mathcal{M}, \tau, z} \left[ \log \left( 1+\frac{p\left(z\right)}{p_\phi\left(z \mid \tau\right)} \sum_{M_i\in \mathcal{M} \backslash\{M\}} \frac{p_\phi\left(z \mid \tau^i\right)}{p\left(z\right)}\right) \right] + \mathrm{log}N \\
&= \mathbb{E}_{\mathcal{M}, \tau, z} \left[ \log \left( \frac{\frac{p_\phi\left(z \mid \tau\right)}{p\left(z\right)}}{\frac{p_\phi\left(z \mid \tau\right)}{p\left(z\right)}+\sum_{M_i\in \mathcal{M} \backslash\{M\}} \frac{p_\phi\left(z \mid \tau^i\right)}{p\left(z\right)}}\right) \right] + \mathrm{log}N \\
&= \mathbb{E}_{\mathcal{M}, \tau, z} \left[ \log \left( \frac{\frac{p_\phi\left(z \mid \tau\right)}{p\left(z\right)}}{\sum_{j=1}^{N} \frac{p_\phi\left(z \mid \tau^i\right)}{p\left(z\right)}}\right) \right] + \mathrm{log}N. \\
\end{align*}
Here, $\mathcal{M} \backslash\{M\}$ is a set of tasks in $\mathcal{M}$ except the task $M$. $(a)$ is derived from the following expression:
\begin{align*}
\begin{split}
\frac{p(z \mid M)}{p(z)} =\int \frac{p(\tau \mid M) p(z \mid \tau, M)}{p(z)} {\rm d} \tau  =\int \frac{p(\tau ) p_\phi(z \mid \tau)}{p(z)} {\rm d} \tau   =\mathbb{E}_\tau \left[\frac{p_\phi(z \mid \tau)}{p(z)}\right].  
\end{split}
\nonumber 
\end{align*}

The inequality at $(b)$ can be proved as follows:
\begin{align*}
&\mathbb{E}_{\tau, z}\left[\log \left(1+\frac{p(z)}{p(z \mid \tau)}(N-1)\right)\right] - \mathbb{E}_{\tau, z}\left[\log \left(\frac{p(z)}{p(z \mid \tau)}N\right)\right] \\
&= \mathbb{E}_{\tau, z}\left[\log \left(\frac{1}{N}\cdot\frac{p(z \mid \tau)}{p(z)} + \frac{N-1}{N} \cdot 1 \right)\right]\\
& \overset{(c)} \geqslant \mathbb{E}_{\tau, z}\left[  \frac{1}{N}\log\frac{p(z \mid \tau)}{p(z)} + \frac{N-1}{N}\log1 \right]\\
&= \frac{1}{N} \mathbb{E}_{\tau, z}\left[ \log\frac{p(z \mid \tau)}{p(z)}
 \right]\\
&= \frac{1}{N} I(z;\tau) \\
&\geqslant 0
\end{align*}
The inequality at (c) is derived from Jensen's Inequality. Thus, the proof is completed.
\end{proof}




\upperb*

\begin{proof}
\begin{align*} 
I(z ; a)&=  \int\int p\left(z, a\right) \log \frac{p\left(z \mid a \right)}{p(z)} {\rm d}z {\rm d}a \\
&= \int\int p\left(z, a\right) \log p\left(z \mid a\right){\rm d}z{\rm d}a-\int p(z) \log p(z) {\rm d}z \\
& \overset{(d)} \leqslant \int \int p\left(z, a\right) \log p\left(z \mid a\right){\rm d} z {\rm d}a-\int p(z) \log t(z){\rm d}z \\
&= \int\int p\left(a\right) p\left(z \mid a\right) \log \frac{p\left(z \mid a\right)}{t(z)} {\rm d}z {\rm d}a  \\
& = \mathbb{E}_a \left[  \int p\left(z \mid a\right) \log \frac{p\left(z \mid a\right)}{t(z)} {\rm d} z  \right]  \\
&= \mathbb{E}_a \left[  D_{\rm KL}\left[p\left(\cdot \mid a\right) \| t(\cdot)\right] \right].
\end{align*}
The inequality at $(d)$ is derived from $D_{\rm KL}\left[p(\cdot) \| t(\cdot)\right] \geq 0$. 
\end{proof}

\subsection{Explanations of the Approximation in Distortion Loss} \label{Explanations of the Approximation in Distortion Loss}

In this section, we further explain the reason for the approximation of $p(z|\tau)/p(z)$ by $S(z, \bar{z}) = \mathrm{exp}(-||z - \bar{z}||^2/{\alpha})$. From Bayesian theory, we know that the bigger $\frac{p(z|\tau)}{p(z)}$ is, the more likely $z$ correlates with $\tau$. In practice, it may be intractable to get the marginal probability $p(z)$. Thus, inspired by InfoNCE, we learn a positive real score function $S(z,\tau)$ to approximate $\frac{p(z|\tau)}{p(z)}$. Given $N$ trajectories $\tau^1, \tau^2, \cdots, \tau^N$ sampled from $N$ different dynamics respectively, \cref{proposition1} tells us that, to maximize $I(z; M)$, the encoder is required to predict the context $z$ with $\tau_i$ as input and $z$ must has least correlation with all $\tau_j\ (j =1,2, \cdots, i-1, i+1, \cdots, N)$. This gives us an intuition that $S(z, \tau)$ should measure the similarity of $z$ and $\tau$. Practically, we implement $S$ with a Gaussian kernel function (though there can be more choices, such as the cosine similarity used by InfoNCE). Moreover, since $\tau$ has an indefinite length, it's difficult to directly calculate the similarity of $z$ and $\tau$. Therefore, in our implementation, we turn to calculate the similarity between $z$ and $\bar{z}$, where $\bar{z}$ is the moving average of all the contexts from the dynamics in which we collect $\tau$.

\section{DORA Testing}
The pseudocodes of the testing phase of DORA is illustrated in Algorithm 2.

\begin{algorithm} \label{DORA testing}
    \caption{DORA Test}
    \begin{algorithmic}
        \STATE {\bfseries Input:} testing tasks $\{M_{i}^{\mathrm{test}}\}_{i = 1}^{N}$, pre-trained context encoder $p_{\phi}$, pre-trained contextual policy $\pi_\theta(s, z)$, episode length $T_s$
        \FOR{each task $M_i \in \{M_{i}^{\mathrm{test}}\}_{i = 1}^{N}$}
            \STATE Initialize a history trajectory $\tau$ of fixed length $H$ with zeros
            \STATE Get init state $s_0$ of $M_i$, append $(s_0, 0)$ to $\tau$
            \FOR{$ 0 \leq t \leq T_s$}
                \STATE Infer $z^t$ from $\tau$ with $p_\phi$
                \STATE Get action $a = \pi_\theta(s, z^t)$
                \STATE Step in env and get the next state $s'$
                \STATE $s \leftarrow s'$
                \STATE Append $(s, a)$ to $\tau$
            \ENDFOR
        \ENDFOR
    \end{algorithmic}
\end{algorithm}

\section{Experiments Details} \label{Appendix: Experiments Details}

\subsection{Task Description} 
We choose several MuJoCo tasks for experiments, including \texttt{HalfCheetah-v3}, \texttt{Walker2d-v3}, \texttt{Hopper-v3}, and \texttt{InvertedDoublePendulum-v2}, which are common benchmarks in offline RL \cite{mujoco}. In \texttt{HalfCheetah}, \texttt{Walker2d}, and \texttt{Hopper}, agents with several degrees of freedom need to learn to move forward as fast as possible. \texttt{InvertedDoublePendulum} originates from the classical control problem \texttt{Cartpole}, in which the agent needs to push the cart left and right to balance the pole on top of the bottom pole.

\subsection{Environment Details} \label{Appendix: Environment Details}
In this section, we provide detailed information about the dynamic-changing environments.

\textbf{Gravity:}  \texttt{Gravity} is a global parameter in MuJoCo, which mainly affects the agent’s fall speed and vertical pressure. The dynamic is sampled by multiplying the default gravity $g_0$ with $1.5^\mu, \mu \sim U[-a, a]$, where $a=1.5, 1.8$ for IID and non-stationary dynamics, respectively. We uniformly sample $\mu$ from $U[-1.8, -1.5]\cup[1.5, 1.8] $ for OOD dynamics.

\textbf{Dof-Damping:} \texttt{Dof-damping} refers to the damping value matrix applied to all degrees of freedom of the agent, serving as linear resistance proportional to velocity. The dynamics are sampled by multiplying the default damping matrix $A_0$ with $1.5^\mu, \mu \sim U[-a, a]$, where $a=1.5, 1.8$ for IID and non-stationary dynamics, respectively. We uniformly sample $\mu$ from $U[-1.8, -1.5]\cup[1.5, 1.8] $ for OOD dynamics.

\textbf{Torso-mass:} \texttt{Torsoz-mass} is the mass of the agent's torso. The dynamic is sampled by multiplying the default mass of torso $m_0$ with $1.5^\mu, \mu \sim U[-a, a]$, where $a=1.5, 1.8$ for IID and non-stationary dynamics, respectively. We uniformly sample $\mu$ from $U[-1.8, -1.5]\cup[1.5, 1.8] $ for OOD dynamics.

Besides, each environment contains 10 tasks for training and 10 tasks for testing for both IID, OOD, and Non-stationary dynamics. 


\subsection{Details of the Offline Dataset} \label{Appendix: Details of Offline Dataset}


\textbf{Maximum and Minimum Returns of Offline Datasets.}
In \cref{table3}, we explicitly provide the maximum and minimum returns of offline datasets, which are used to calculate the normalized return in the policy evaluation. Specifically, the evaluation return $x$ for a particular is normalized using the formula $\frac{x - x_{\mathrm{min}}}{x_{\mathrm{max}}-   x_{\mathrm{min}}} \times 100 $, where $x_{\mathrm{max}}$ and $x_{\mathrm{min}}$ denote the  corresponding maximum and minimum returns of each dataset. 

\begin{table*}[htbp]
  \centering
  \caption{Maximum and minimum returns of offline datasets.}
  \small{
    \begin{tabular}{ccccccc}
    \toprule
    \multicolumn{1}{c}{} & \multicolumn{1}{c}{Cheetah-gravity} & \multicolumn{1}{c}{Cheetah-dof} & \multicolumn{1}{c}{Cheetah-torso\_mass} & \multicolumn{1}{c}{Hopper-gravity} & \multicolumn{1}{c}{Walker-graavity} & \multicolumn{1}{c}{Pendulum-gravity} \bigstrut\\
    \midrule
    Maximum &  9218.48 & 8703.55 & 10060.59 & 3536.97 & 4630.10 & 9351.94\\
    Minimal & -460.33 & -613.09 & -479.86 & -1.22 & -110.50 & 26.71\\
    
    \bottomrule
    \end{tabular}%
  }
  \label{table3}%
\end{table*}%

\subsection{Introduction of Baselines}  \label{Appendix: intro of baselines}
We compare DORA with 4 baselines which are listed below. 

\textbf{CORRO}. CORRO \cite{CORRO} designs a bi-level task encoder where the transition encoder is optimized by contrastive task learning and the aggregator encoder gathers all representations. CORRO uses Generative Modeling and Reward Randomization to generate negative pairs for contrastive learning. In our setting, the reward function remains the same across tasks thus we adopt the approach of Generative Modeling to compare.

\textbf{FOCAL}. FOCAL \cite{FOCAL} uses distance metric loss to learn a deterministic contextual encoder, which is also under the paradigm of contrastive learning. Additionally, FOCAL combines BRAC \cite{BRAC} to constrain the bootstrapping error.

\textbf{Prompt-DT}. Prompt-DT \cite{Prompt-DT} employs the Transformer to tackle OMRL tasks. It utilizes a short segment of task trajectory containing task-specific information as the prompt input, guiding the transformer to model trajectories from different tasks. For each task dataset, we select the top 5 trajectories with the highest cumulative rewards to construct the task prompt. The remaining trajectories are then used for training.

\textbf{Offline ESCP}. ESCP \cite{ESCP} is an online meta RL approach that proposes variance minimization and Relational Matrix Determinant Maximization in optimizing the encoder to adapt to environment sudden changes. In order to fairly compare ESCP with other algorithms in offline settings, we develop an offline version of ESCP, denoted as offline ESCP. In this variant, we modify the online interaction process with the environment to sample trajectories from the offline dataset.

\subsection{Configurations}
The details of the important configurations and hyper-parameters used to produce the experimental results in this paper are listed in \cref{table4}. Regarding the encoder model, we utilized a linear layer with 128 hidden units, a GRU network with 64 hidden units, and a linear layer with 64 hidden units for parameterization. The encoder's output is scaled using a $\mathrm{Tanh}$ function.

\begin{table*}[htbp]
  \centering
  \caption{Configurations and hyper-parameters used in offline encoder training.}
  \small{
    \begin{tabular}{ccccccc}
    \toprule
    \multicolumn{1}{c}{Configurations} & \multicolumn{1}{c}{Cheetah-gravity} & \multicolumn{1}{c}{Cheetah-dof} & \multicolumn{1}{c}{Cheetah-torso} & \multicolumn{1}{c}{Hopper-gravity} & \multicolumn{1}{c}{Walker-gravity} & \multicolumn{1}{c}{Pendulum-gravity} \bigstrut\\
    \midrule
     Debias loss weight& 0.2  & 1.0 & 1.0 & 0.2 & 1.0 & 1.0\\
     Distortion loss weight & 1.0 & 1.0 & 1.0 & 1.0 & 1.0 & 1.0\\
    History length  & 8 & 8 & 8 & 8 & 8 & 8 \\
    Latent space dim & 2 & 2 & 2 & 2 & 2 & 2\\
    Batch size & 256 & 256 & 256 & 256 & 256 & 256\\
    Learning rate & 3e-4 & 3e-4 & 3e-4 & 3e-4 & 3e-4 & 3e-4\\
    Training steps & 2000000 & 2000000 & 2000000 & 2000000 & 2000000 & 400000\\
    \bottomrule
    \end{tabular}%
  }
  \label{table4}%
\end{table*}%

\section{Representation Visualization of DORA} \label{Appendix: Representation Visualization of DORA}

\begin{figure}[ht] 
  \centering
  \subcaptionbox{Cheetah-dof}{
    \includegraphics[width=0.18\linewidth, trim = 0cm 0.6cm 0.5cm 1.0cm, clip]{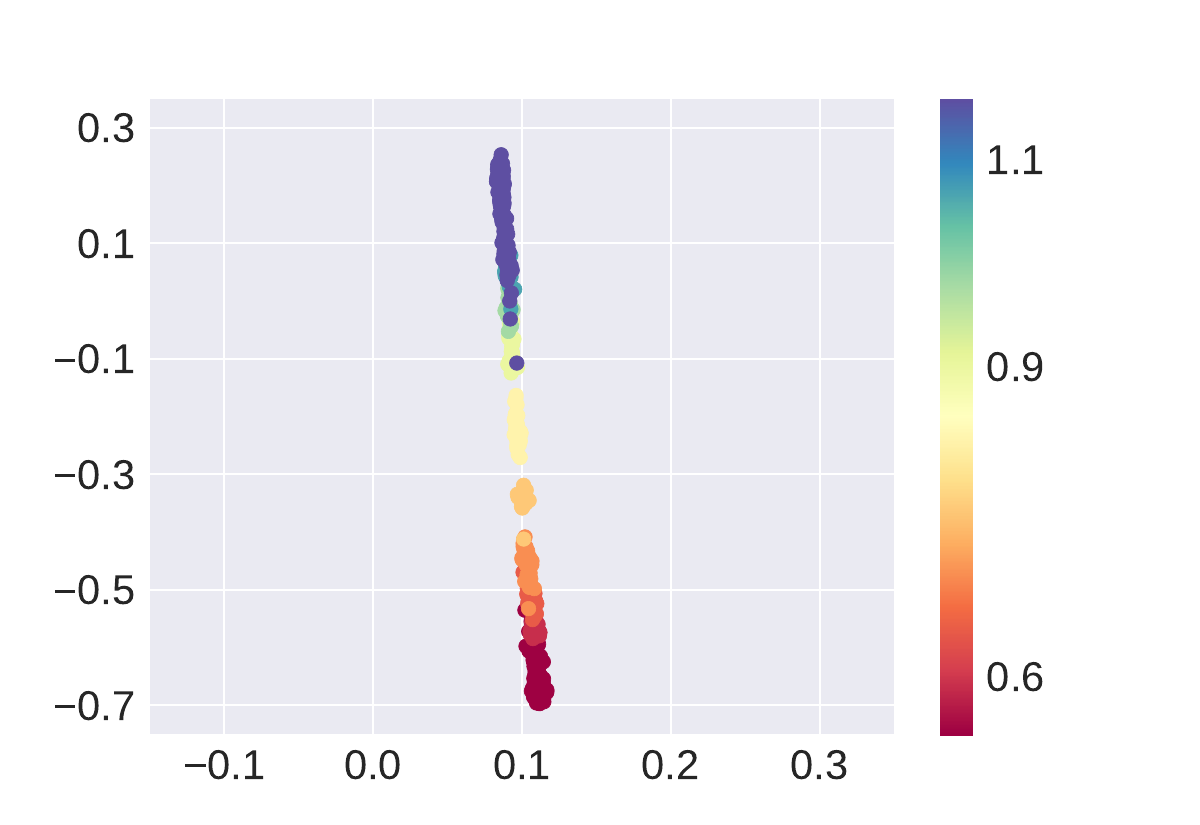}
  }
  \hspace{0.2mm} 
  \subcaptionbox{Cheetah-torso}{
    \includegraphics[width=0.18\linewidth, trim = 0cm 0.6cm 0.5cm 1.0cm, clip]{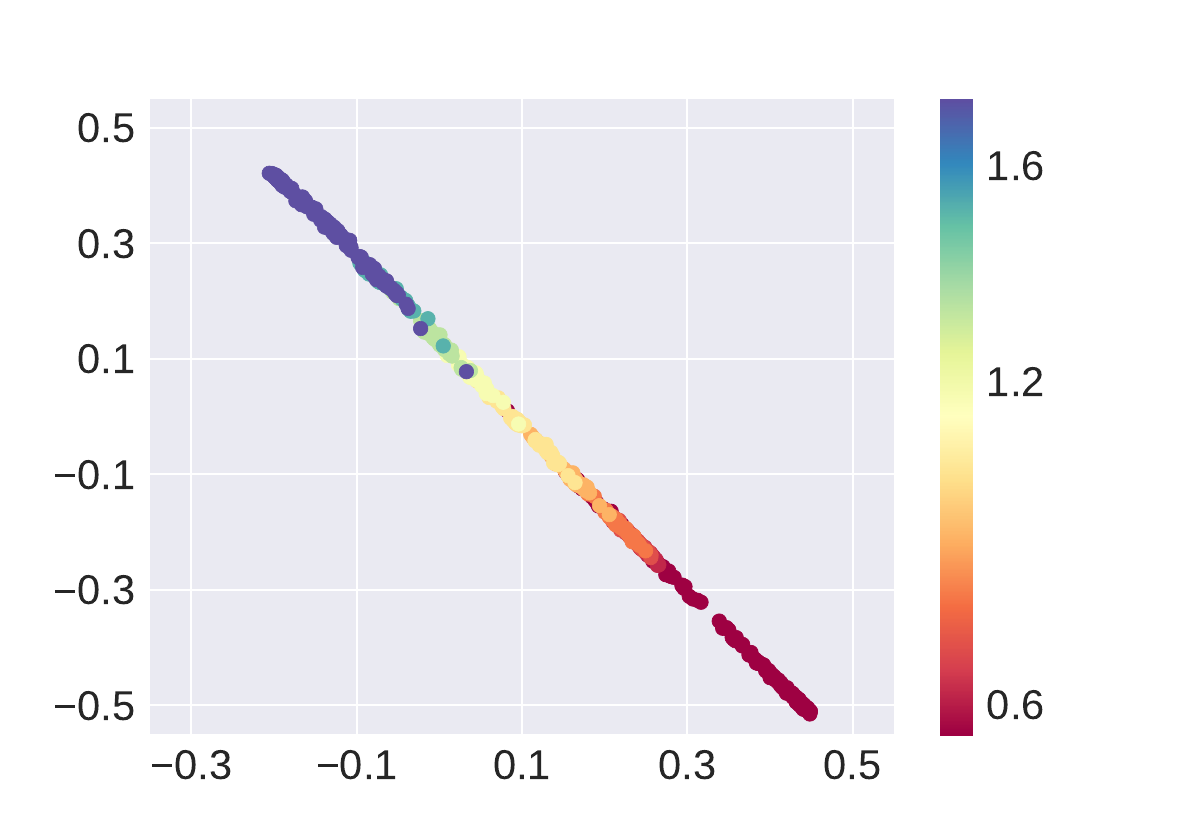}
  }
  \hspace{0.2mm} 
  \subcaptionbox{Pendulum-gravity}{
    \includegraphics[width=0.18\linewidth, trim = 0cm 0.6cm 0.5cm 1.0cm, clip]{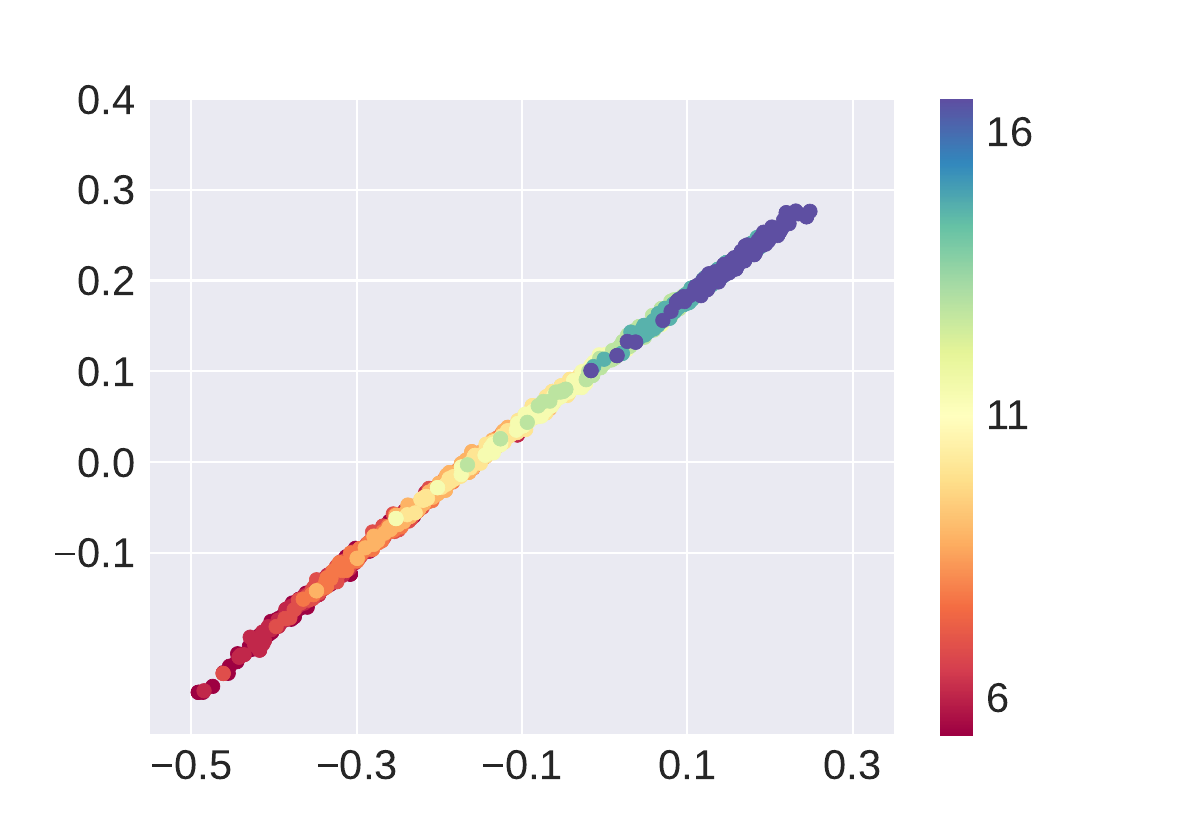}
  } 
  \hspace{0.2mm} 
  \subcaptionbox{Walker-gravity}{
    \includegraphics[width=0.18\linewidth, trim = 0cm 0.6cm 0.5cm 1.0cm, clip]{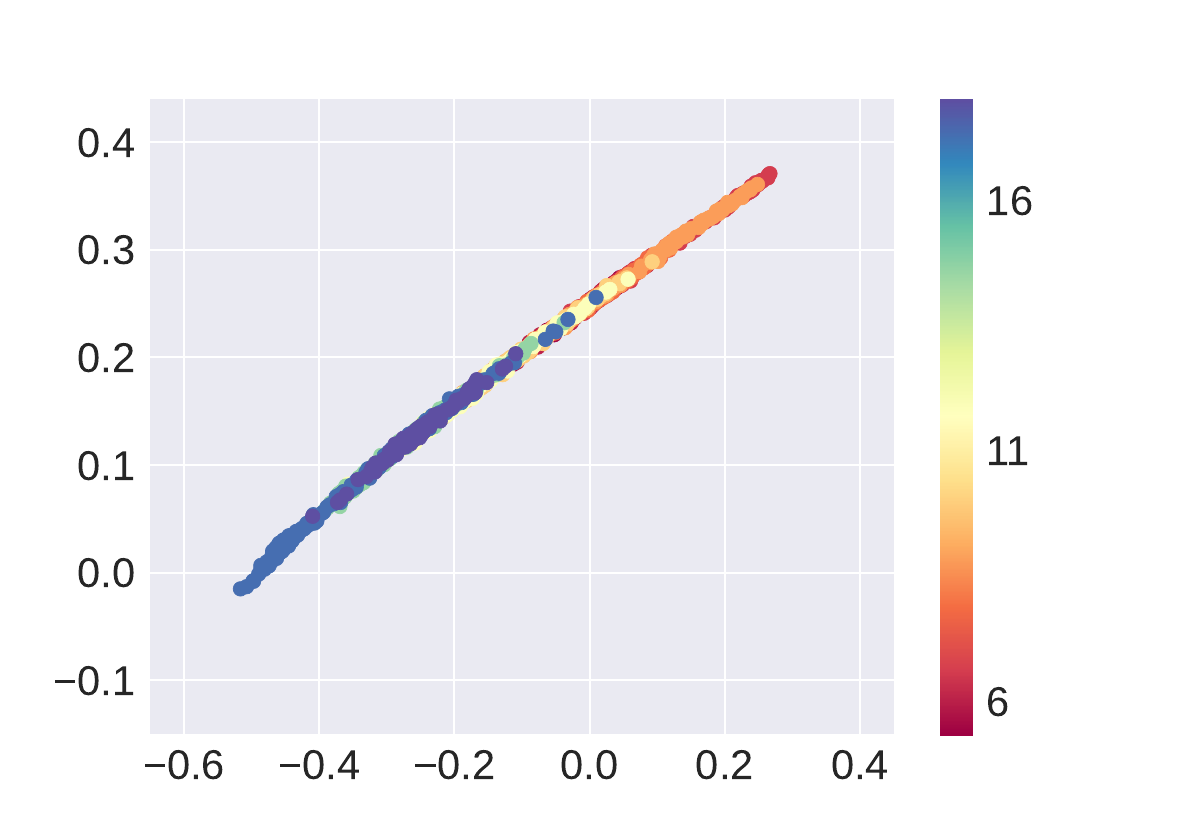}
  } 
  \hspace{0.2mm} 
  \subcaptionbox{Hopper-gravity}{
    \includegraphics[width=0.18\linewidth, trim = 0cm 0.6cm 0.5cm 1.0cm, clip]{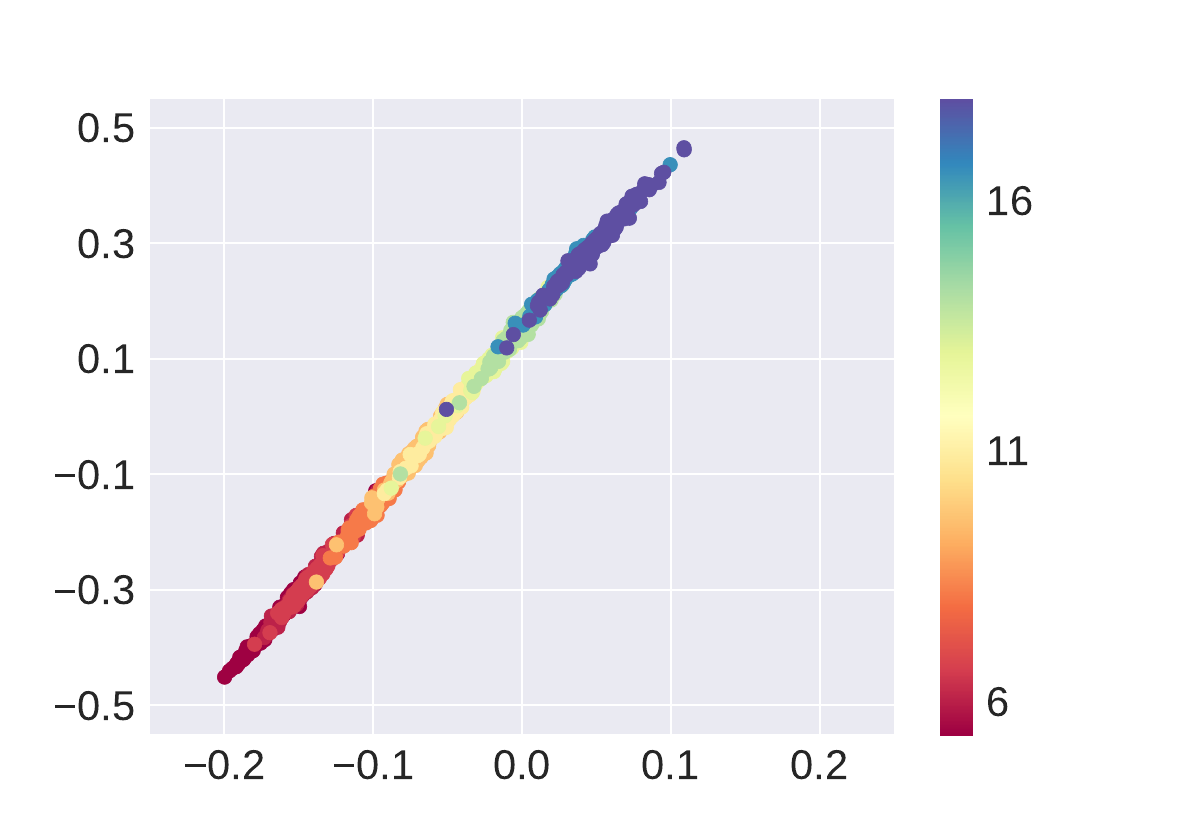}
  } 
\caption{Representation visualization of DORA in all 6 MuJoCo tasks.}
\label{Appendix: Representation Visualization}
\end{figure}
We present the representation visualization figures of all 6 MuJoCo tasks in \cref{Appendix: Representation Visualization}. With the color indicating the real parameters of dynamics, we find that the data points are sorted sequentially based on the real values of the varying parameters, which shows that the task representations capture the information of different task dynamics. 

Moreover, we visualize the encodings of DORA in \texttt{Cheetah-gravity} with OOD dynamics, shown in \cref{Fig: Appendix OOD Visulization}. It is evident to find that the representations in OOD dynamics (\cref{Fig: Appendix visual in OOD}) are situated at both ends of the encodings in IID dynamics (\cref{Fig: Appendix visual in IID}). Such visualization results indicate that DORA learns generalizable representations, which can be extended to unseen OOD dynamics.

\begin{figure}[ht] 
  \centering
  \begin{subfigure}[b]{0.22\textwidth}
    \includegraphics[width=\textwidth, trim = 1cm 0.5cm 2cm 1.0cm, clip]{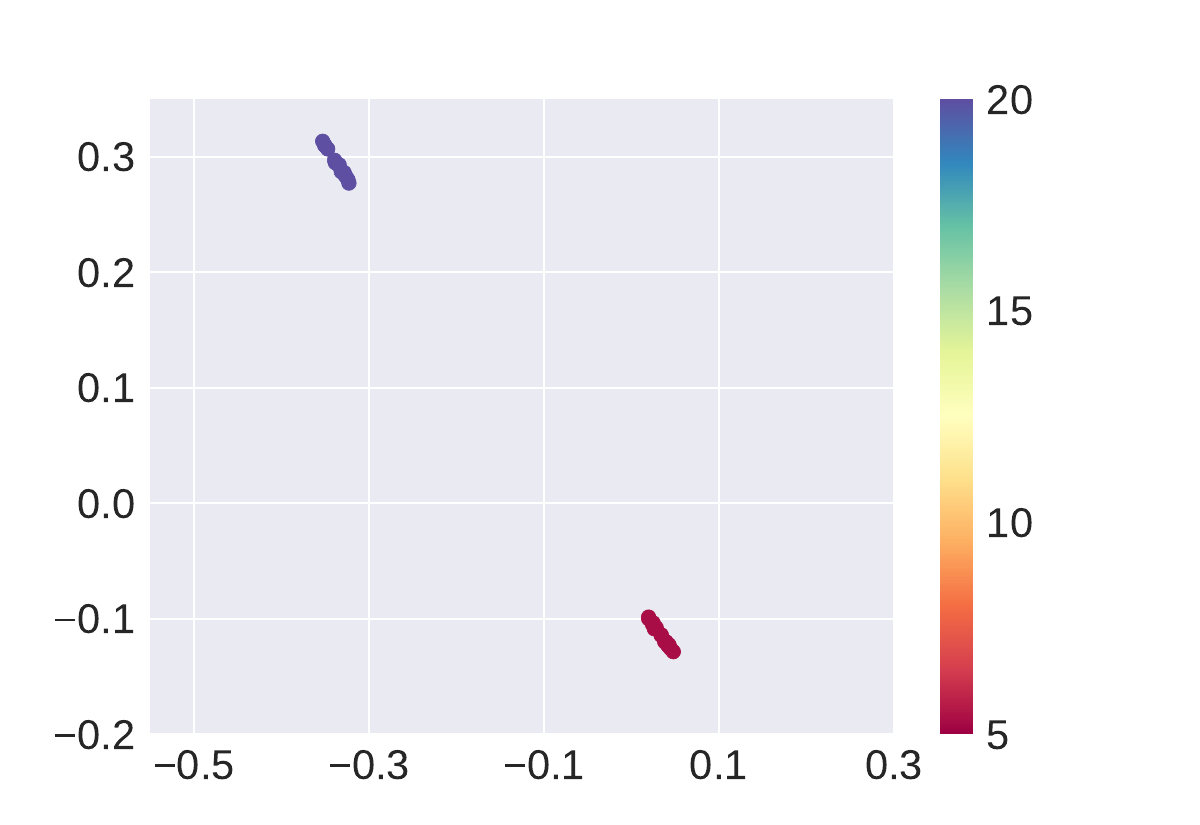}
    \caption{In OOD dynamics}
    \label{Fig: Appendix visual in OOD}
  \end{subfigure}
  \hspace{10mm}
  \begin{subfigure}[b]{0.22\textwidth}
    \includegraphics[width=\textwidth, trim =  1cm 0.5cm 2cm 1.0cm, clip]{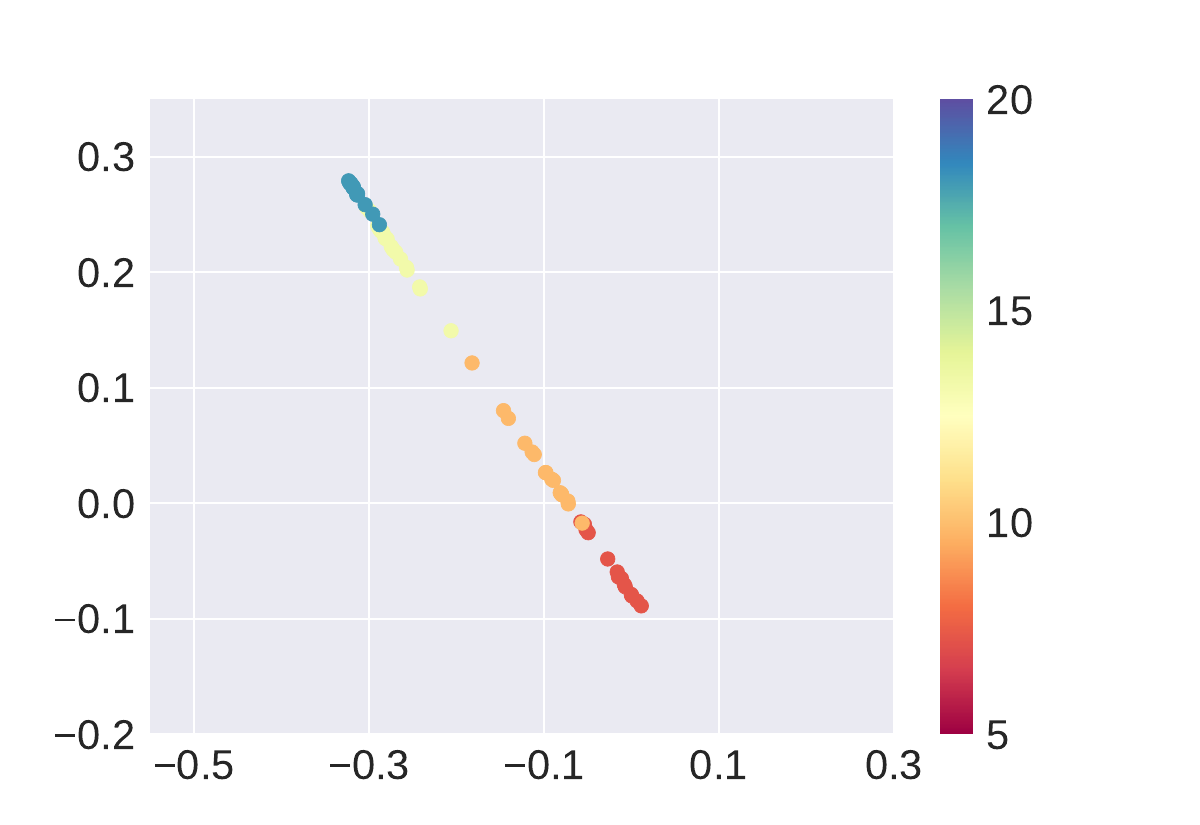}
    \caption{In IID dynamics}
    \label{Fig: Appendix visual in IID}
  \end{subfigure}
\caption{Representation visulization in \texttt{Cheetah-gravity}. \textbf{Left:} In OOD dynamics. \textbf{Right:} In IID dynamics. }
\label{Fig: Appendix OOD Visulization}
\end{figure}

\section{Experiments in the Setting of OMRL}

In the general OMRL setting, the algorithms are permitted to use a certain policy to collect context before evaluation. We compare the performance of DODA and other baselines in such a setting and the results are shown in figure \ref{training_curve}. Benefiting from the effective task representations, DORA still exhibits a remarkable performance and outperforms the baselines in 5 of the total 6 MuJoCo tasks. 

\begin{figure*}[htbp]
  \centering
    \includegraphics[width=0.78\linewidth, trim = 3cm 0.5cm 3cm 2.2cm, clip]{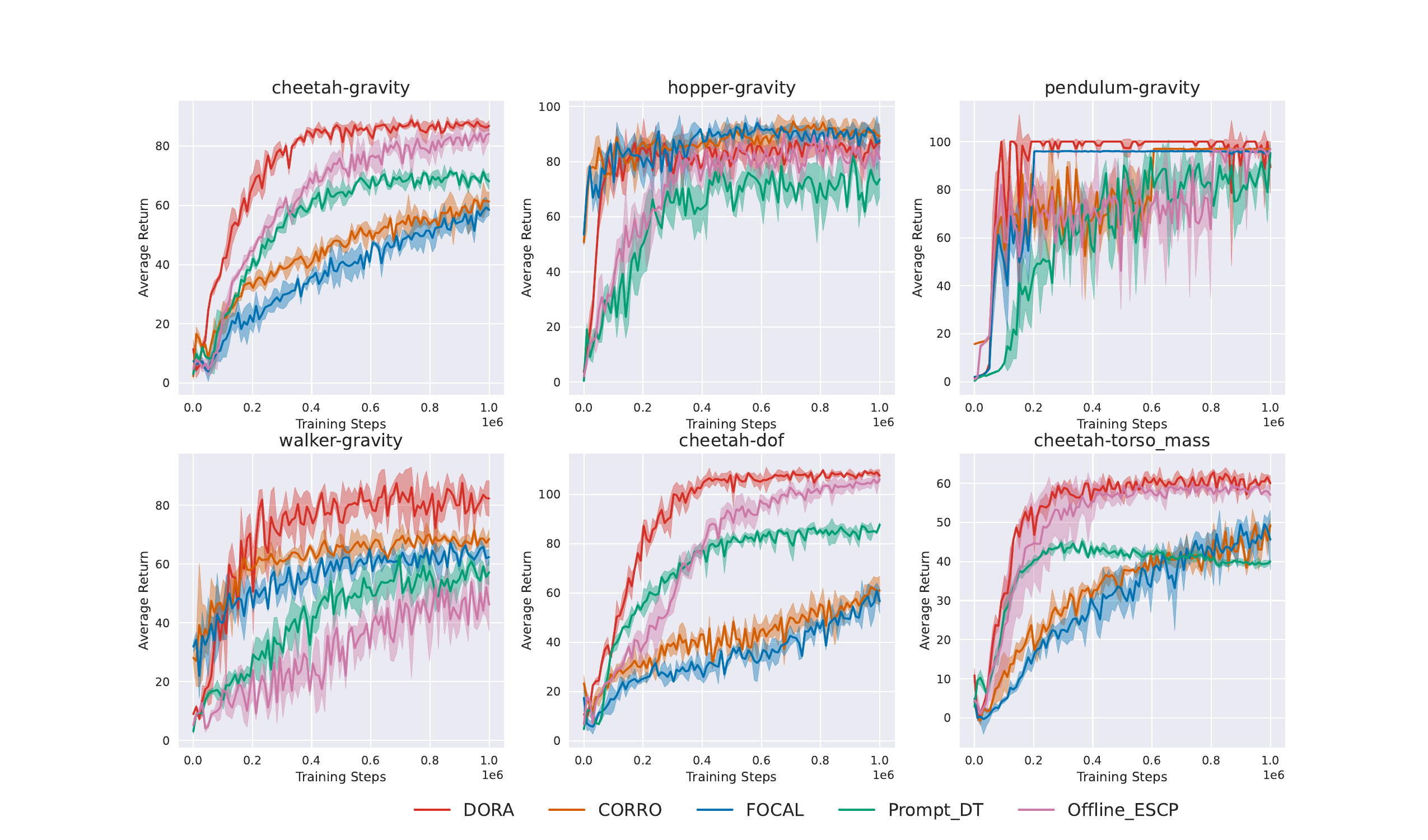}
  \caption{ Test returns of DORA against the baselines in IID dynamics in the general OMRL setting. During meta-testing, a pre-collected context can be utilized to infer task representations before evaluation.}
  \label{training_curve}
\end{figure*}

\section{Further Ablation Studies} \label{Appendix: Ablation}

\subsection{Ablation Studies on the Losses} 

To visually illustrate the impact of the distortion loss and debias loss in DORA, we take the visualization results on \texttt{Cheetah-gravity} as an example. \cref{Fig: Without Distortion Loss} shows that the encoder can not extract informative representations without distortion loss. In addition, without the debias loss, representations inferred from tasks of similar dynamics no longer exhibit correlations as shown in \cref{Fig: Without Debias Loss}.

\subsection{Ablation Studies on $\beta$} \label{Appendix: Ablation Studies on beta}
We conducted sensitivity tests on the hyper-parameter $\beta$ on 3 MuJoCo tasks, and the results indicate that excessively large or small values of $\beta$ lead to a decline in policy performance. Although in \texttt{Pendulum-gravity}, the case with $\beta=0.2$ shows better average performance, the policy performance sharply drops when $\beta=0$ (In \cref{Fig: ablation_loss}). These experimental results, on the one hand, demonstrate that when $\beta$ is too small, the encoder struggles to debias representations from the behavior policy. On the other hand, when $\beta$ is too large, the diminishing effect of the distortion loss during the optimization process leads to a decrease in policy performance.

\begin{figure}[htbp] 
  \centering
  \begin{subfigure}[b]{0.255\textwidth}
    \includegraphics[width=\textwidth, trim = 1cm 0.5cm 2cm 1.0cm, clip]{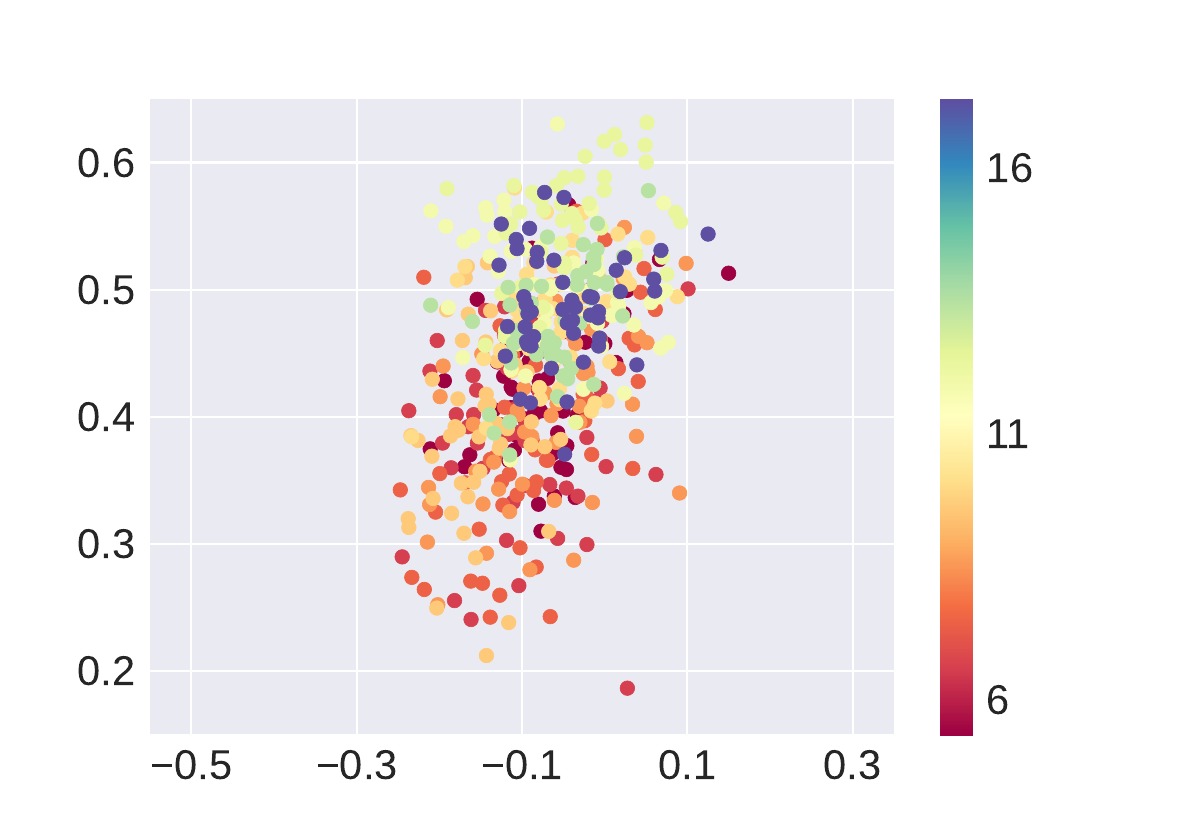}
    \caption{Without Distortion Loss}
    \label{Fig: Without Distortion Loss}
  \end{subfigure}
  \hspace{5mm}
  \begin{subfigure}[b]{0.255\textwidth}
    \includegraphics[width=\textwidth, trim =  1cm 0.5cm 2cm 1.0cm, clip]{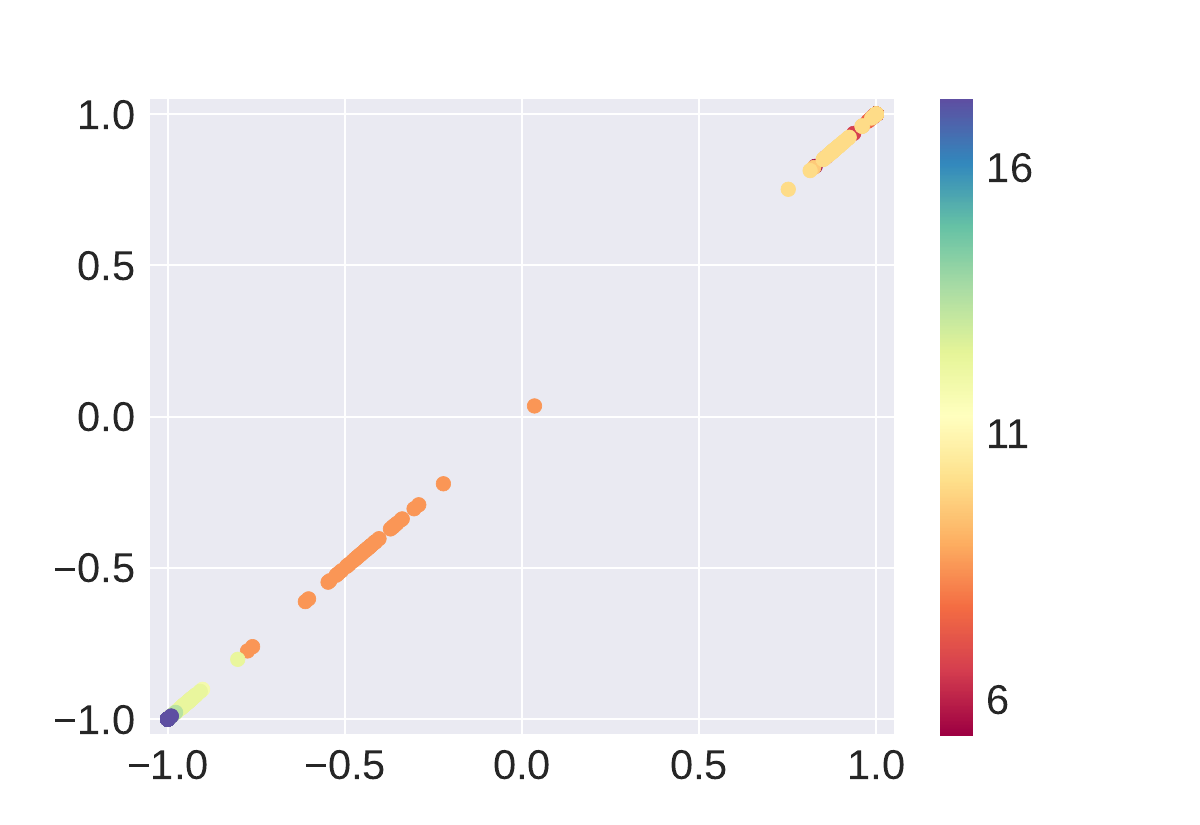}
    \caption{Without Debias Loss}
    \label{Fig: Without Debias Loss}
  \end{subfigure}
  \hspace{5mm}
  \begin{subfigure}[b]{0.3\textwidth}
    \includegraphics[width=\textwidth, trim =  1cm 0.5cm 2cm 1.0cm, clip]{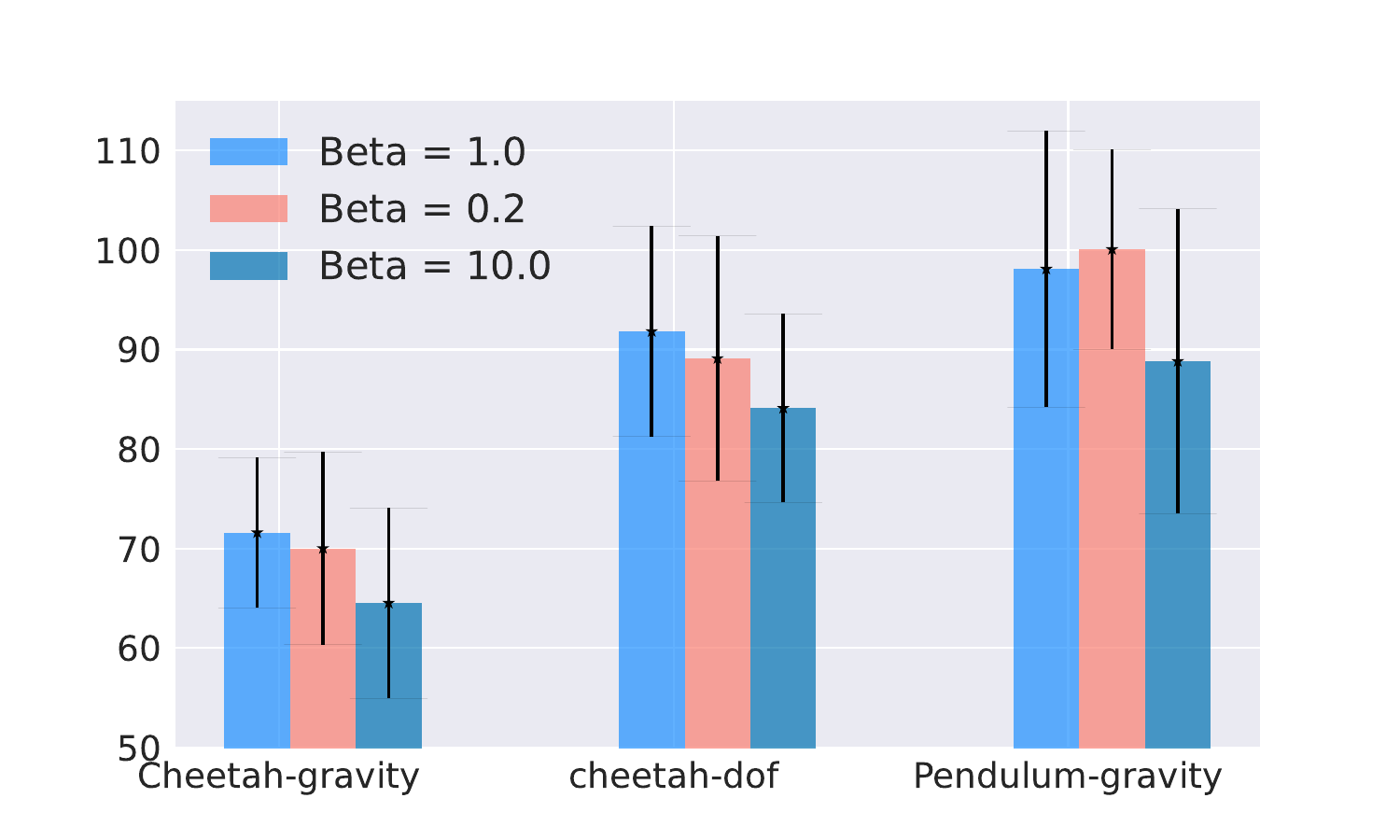}
    \caption{Ablation Studies on $\beta$}
    \label{Fig: ablation on beta}
  \end{subfigure}
\caption{Further ablation studies. \textbf{Left:}  DORA without the distortion loss in \texttt{Cheetah-gravity}.  \textbf{Middle:} DORA without debias loss in \texttt{Cheetah-gravity}. \textbf{Right:} Average normalized return of DORA with $\beta$ = 0.2, 1.0, and 10.0 over 5 random seeds. The error bar stands for the standard deviation.  }
\label{Fig: Appendix ablation visualization}
\end{figure}

\subsection{Performance with Less Offline Data Available} \label{Appendix: Performance with Less Offline Data Available}

We conducted experiments to test the performance of DORA under the conditions of halving the size of each single-task dataset (i.e. 10e5 transitions per single-task dataset) and find DORA still outperforms the other baselines in \cref{table5}.

\begin{table*}[htbp]
  \centering
  \caption{Performance with Less Offline Data Available on \texttt{Cheetah-gravity} in Non-stationary Dynamics.}
  \small{
    \begin{tabular}{cccccc}
    \toprule
    \multicolumn{1}{c}{} & \multicolumn{1}{c}{FOCAL} & \multicolumn{1}{c}{CORRO} & \multicolumn{1}{c}{Prompt-DT } & \multicolumn{1}{c}{Offline ESCP} & \multicolumn{1}{c}{DORA}  \bigstrut\\
    \midrule
    Half Datastes & 47.43 $\pm$ 4.24 & 51.93 $\pm$ 4.84 & 32.85 $\pm$ 7.66 & 51.75 $\pm$ 4.36 & \textbf{66.28 $\pm$ 7.85} \\
    Full Datasets & 49.96 $\pm$ 6.41 & 54.08 $\pm$ 8.54 & 23.52 $\pm$ 7.75 & 65.42 $\pm$ 6.54 & \textbf{71.61 $\pm$ 7.56} \\
    \bottomrule
    \end{tabular}%
  }
  \label{table5}%
\end{table*}%

Besides, we conducted experiments of varying the number of training tasks to 5, 8, and 10. The experimental results are shown in \cref{table6}. We observe that as the number of tasks decreases, the performance of all algorithms declines. This is because fewer training tasks can affect the algorithms' generalization performance in non-stationary dynamics. 

\begin{table*}[htbp]
  \centering
  \caption{Performance with Less Traning Tasks on \texttt{Cheetah-gravity} in Non-stationary Dynamics.}
  \small{
    \begin{tabular}{cccc}
    \toprule
    \multicolumn{1}{c}{Number of Traning Tasks } & \multicolumn{1}{c}{5} & \multicolumn{1}{c}{8} & \multicolumn{1}{c}{10}  \bigstrut\\
    \midrule
    DORA's Performance & 54.92 $\pm$ 17.71 & 65.48 $\pm$ 10.58 & 71.61 $\pm$ 7.56  \\
    \bottomrule
    \end{tabular}%
  }
  \label{table6}%
\end{table*}%

\subsection{Ablation Studies on Offline RL Algorithms} \label{Appendix: Ablation Studies on Offline RL Algorithms}

Besides CQL, we train the contextual policy of DORA with different offline RL algorithms, including a model-based method COMBO\cite{COMBO} and a model-free method EDAC\cite{EDAC}. The results are shown in \cref{table7}.

Although DORA(CQL) is not the best performing among several versions, we use CQL as the general offline RL algorithm in our framework for 2 reasons. Firstly, CQL is well-known and widely applied in offline RL. Secondly, this algorithm is relatively simple and easy to implement. Thus, we can easily compare the performance of various baselines upon it.

\begin{table*}[htbp]
  \centering
  \caption{Performance with Different Offline RL Algorithms in Non-stationary Dynamics.}
  \small{
    \begin{tabular}{cccccc}
    \toprule
    \multicolumn{1}{c}{} & \multicolumn{1}{c}{DORA(CQL)} & \multicolumn{1}{c}{DORA(EDAC)} & \multicolumn{1}{c}{DORA(COMBO)}   \bigstrut\\
    \midrule
    cheetah-gravity & 71.61 $\pm$ \,\,\,7.56 & \textbf{73.58 $\pm$ \,\,\,5.44} & \,\,\,69.16 $\pm$ 7.36 \\
    cheetah-dof & 91.83 $\pm$ 10.57 & 89.89 $\pm$ 
 \,\,\,5.85 & \textbf{\,\,\,93.04 $\pm$ 7.47} \\
    pendulum-gravity & 97.11 $\pm$ 16.91 & 93.12 $\pm$ 15.35 & \textbf{100.08 $\pm$ 0.05} \\
    \bottomrule
    \end{tabular}%
  }
  \label{table7}%
\end{table*}%

\subsection{Performance with Different Changing Rate of Non-stationary Dynamics} \label{Appendix: Performance with Different Changing Rate of Non-stationary Dynamics}

In order to test DORA's performance with different changing rates of non-stationary dynamics, we change the physical parameters of the environments every 10, 30, and 50 timesteps. The results in \cref{table8} indicate that DORA can adapt to frequent changes in dynamics.

\begin{table*}[htbp]
  \centering
  \caption{Performance with Different Changing Rate of Non-stationary Dynamics.}
  \small{
    \begin{tabular}{cccccc}
    \toprule
    \multicolumn{1}{c}{
    Tasks \textbackslash{} Changing per Steps } & \multicolumn{1}{c}{10} & \multicolumn{1}{c}{30} & \multicolumn{1}{c}{50}   \bigstrut\\
    \midrule
    cheetah-gravity & 64.28 $\pm$ \,\,\,4.01 & 66.62 $\pm$ \,\,\,5.78 & 71.61 $\pm$ \,\,\,7.56  \\
    cheetah-dof & 85.09 $\pm$ \,\,\,5.63 & 87.44 $\pm$ \,\,\,6.48 & 91.83 $\pm$ 10.57 \\
    pendulum-gravity & 96.89 $\pm$ 14.56 & 96.49 $\pm$ 12.44 & 97.11 $\pm$ 16.91 \\
    \bottomrule
    \end{tabular}%
  }
  \label{table8}%
\end{table*}%


\end{document}